\documentclass[11pt]{article}  

\usepackage{etex}

\usepackage{fullpage}

\usepackage{amsmath}
\usepackage{amssymb}
\usepackage{amsthm}
\usepackage{array}
\usepackage{bbm}
\usepackage{cancel}
\usepackage{cmap}
\usepackage{enumerate}
\usepackage{enumitem}
\usepackage{fancyhdr}
\usepackage{mathdots}
\usepackage{mathtools}
\usepackage{mathrsfs}
\usepackage{hyperref}
\usepackage{stackrel}
\usepackage{stmaryrd}
\usepackage{tabularx}
\usepackage{tikz}
\usepackage{ctable}
\usepackage{titlesec}
\usepackage{titletoc}
\usepackage{url}
\usepackage{verbatim}
\usepackage{wasysym}
\usepackage{wrapfig}
\usepackage{yhmath}
\usepackage[all,cmtip]{xy}

\usepackage{hyperref}
\usepackage[%
	firstinits=true,
	doi=false,
	url=false,
	isbn=false,
	backend=biber,
	style=alphabetic,hyperref]{biblatex}

\usetikzlibrary{calc,trees,positioning,arrows,chains,shapes.geometric,%
    decorations.pathreplacing,decorations.pathmorphing,shapes,%
    matrix,shapes.symbols,shadows,fadings}

\newtheorem{thm}{Theorem}[section]
\newtheorem*{thm*}{Theorem}
\newtheorem{prb}[thm]{Problem}
\newtheorem*{prb*}{Problem}

\newtheorem*{ax*}{Axiom}

\newtheorem*{clm*}{Claim}

\newtheorem*{conj*}{Conjecture}
\newtheorem{cor}[thm]{Corollary}
\newtheorem{df}[thm]{Definition}
\newtheorem*{df*}{Definition}

\newtheorem*{ex*}{Example}

\newtheorem{lem}[thm]{Lemma}
\newtheorem*{lem*}{Lemma}

\newtheorem*{pos*}{Postulate}

\newtheorem*{pr*}{Proposition}

\newtheorem*{qu*}{Question}

\newtheorem*{rem*}{Remark}
\def\shownotes{1}  \ifnum\shownotes=1
\newcommand{\authnote}[2]{$\ll$\textsf{\footnotesize #1 notes: #2}$\gg$}
\else
\newcommand{\authnote}[2]{}
\fi

\newcommand{\C}[0]{\mathbb{C}}

\newcommand{\E}[0]{\mathbb{E}}
\newcommand{\EE}[0]{\mathop{\mathbb E}}

\newcommand{\cH}[0]{\mathcal H}

\newcommand{\cL}[0]{\mathscr{L}}
\newcommand{\N}[0]{\mathbb{N}}

\newcommand{\Pj}[0]{\mathbb{P}}

\newcommand{\R}[0]{\mathbb{R}}

\newcommand{\Z}[0]{\mathbb{Z}}
\newcommand{\one}[0]{\mathbbm{1}}

\newcommand{\ga}[0]{\gamma}
\newcommand{\Ga}[0]{\Gamma}
\newcommand{\de}[0]{\delta}

\newcommand{\ep}[0]{\varepsilon}

\newcommand{\la}[0]{\lambda}

\newcommand{\rh}[0]{\rho}

\newcommand{\om}[0]{\omega}
\newcommand{\Om}[0]{\Omega}
\newcommand{\si}[0]{\sigma}
\newcommand{\Si}[0]{\Sigma}
\newcommand{\ze}[0]{\zeta}

\newcommand{\iy}[0]{\infty}

\newcommand{\rc}[1]{\frac{1}{#1}}
\newcommand{\prc}[1]{\pa{\rc{#1}}}

\newcommand{\fc}[2]{\frac{#1}{#2}}
\newcommand{\sfc}[2]{\sqrt{\frac{#1}{#2}}}
\newcommand{\pf}[2]{\pa{\frac{#1}{#2}}}

\newcommand{\ab}[1]{\left| {#1} \right|}

\newcommand{\ba}[1]{\left[ {#1} \right]}
\newcommand{\bc}[1]{\left\{ {#1} \right\}}

\newcommand{\ce}[1]{\left\lceil {#1}\right\rceil}

\newcommand{\pa}[1]{\left( {#1} \right)}

\newcommand{\ve}[1]{\left\Vert {#1}\right\Vert}

\newcommand{\set}[2]{\left\{{#1}:{#2}\right\}}

\newcommand{\ol}[1]{\overline{#1}}

\newcommand{\ub}[2]{\underbrace{#1}_{#2}}

\newcommand{\wh}[1]{\widehat{#1}}

\newcommand{\amin}{\operatorname{argmin}}

\newcommand{\Cov}{\operatorname{Cov}}

\newcommand{\Tr}[0]{\operatorname{Tr}}

\providecommand{\cal}[1]{\mathcal{#1}}
\renewcommand{\cal}[1]{\mathcal{#1}}

\newcommand{\pull}[9]{
#1\ar@/_/[ddr]_{#2} \ar@{.>}[rd]^{#3} \ar@/^/[rrd]^{#4} & &\\
& #5\ar[r]^{#6}\ar[d]^{#8} &#7\ar[d]^{#9} \\}

\newcommand{\cmp}[9]{
\xymatrix{
#1 \ar[r]^{#4}{#5} \ar@/_2pc/[rr]^{#8}_{#9} & #2 \ar[r]^{#6}_{#7} & #3
}
}

\newcommand{\ha}[1]{\ar@{^(->}[#1]}
\newcommand{\ls}[1]{\ar@{-}[#1]}
\newcommand{\sj}[1]{\ar@{->>}[#1]}
\newcommand{\aq}[1]{\ar@{=}[#1]}
\newcommand{\acir}[1]{\ar@{}[#1]|-{\textstyle{\circlearrowright}}}
\newcommand{\acil}[1]{\ar@{}[#1]|-{\textstyle{\circlearrowleft}}}
\newcommand{\ard}[1]{\ar@{.>}[#1]}
\newcommand{\mt}[1]{\ar@{|->}[#1]}
\newcommand{\inm}[1]{\ar@{}[#1]|-{\in}}
\newcommand{\inr}{\ar@{}[d]|-{\rotatebox[origin=c]{-90}{$\in$}}}
\newcommand{\inl}{\ar@{}[u]|-{\rotatebox[origin=c]{90}{$\in$}}}

\newcommand{\sumo}[2]{\sum_{#1=1}^{#2}}
\newcommand{\sumz}[2]{\sum_{#1=0}^{#2}}

\newcommand{\coltwo}[2]{
\begin{pmatrix}
{#1}\\
{#2}
\end{pmatrix}}

\newcommand{\matt}[4]{
\begin{pmatrix}
{#1}&{#2}\\
{#3}&{#4}
\end{pmatrix}
}

\newcommand{\beq}[1]{\begin{equation}\llabel{#1}}
\newcommand{\eeq}[0]{\end{equation}}
\newcommand{\bal}[0]{\begin{align*}}
\newcommand{\eal}[0]{\end{align*}}
\newcommand{\ban}[0]{\begin{align}}
\newcommand{\ean}[0]{\end{align}}

\newcommand{\fixme}[1]{{\color{red}#1}}
\newcommand{\llabel}[1]{\label{#1}\text{\fixme{\tiny#1}}}

\newcommand{\arxiv}[1]{\url{http://www.arxiv.org/abs/#1}}

\allowdisplaybreaks[2]

\DeclareFontFamily{U}{wncy}{}
    \DeclareFontShape{U}{wncy}{m}{n}{<->wncyr10}{}
    \DeclareSymbolFont{mcy}{U}{wncy}{m}{n}
    \DeclareMathSymbol{\Sh}{\mathord}{mcy}{"58} 
\newcommand{\citep}[1]{\cite{#1}}
\usepackage{algorithm}
\usepackage{algorithmic}

\usepackage{etoolbox}
\newtoggle{conference}
\togglefalse{conference}

\newcommand{\hunr}[0]{h_{\text{unr}}}
\newcommand{\hunrs}[0]{h_{\text{unr}}^*}

\newcommand{\boldorpara}[1]{\iftoggle{conference}{\textbf{#1}}{\paragraph{#1}}}

\begin{document}

\title{Robust guarantees for learning an autoregressive filter}

\author{Holden Lee\thanks{Princeton University, Mathematics Department \texttt{holdenl@math.princeton.edu}}, Cyril Zhang\thanks{Princeton University, Computer Science Department \texttt{cyril.zhang@cs.princeton.edu}}
}

\date{\today}
\maketitle

\begin{abstract}%
The optimal predictor for a linear dynamical system (with hidden state and Gaussian noise) takes the form of an autoregressive linear filter, namely the Kalman filter. However, a fundamental problem in reinforcement learning and control theory is to make optimal predictions in an unknown dynamical system. 
To this end, we take the approach of directly learning an autoregressive filter for time-series prediction under unknown dynamics. Our analysis differs from previous statistical analyses in that we regress not only on the inputs to the dynamical system, but also the outputs, which is essential to dealing with process noise. 
The main challenge is to estimate the filter under worst case input (in $\mathcal H_\infty$ norm), for which we use an $L^\infty$-based objective rather than ordinary least-squares. 
For learning an autoregressive model, our algorithm has optimal sample complexity in terms of the rollout length, which does not seem to be attained by naive least-squares.
\end{abstract}

\section{Introduction}

The problem of estimating the hidden state and outputs of a known linear dynamical system (LDS), given the inputs and observations, is a well-studied problem in control theory \citep{kamen1999introduction}. In the case of Gaussian noise, this problem is completely solved by the Kalman filter \cite{kalman1960new}, which recursively propagates the optimal linear estimator for the hidden state. 
When the recursion for the estimator is unrolled, the Kalman filter is seen to be a linear autoregressive filter: it predicts the system's next output as a linear combination of the system's past ground-truth outputs.

However, when the LDS is \emph{unknown}, optimal filtering is a much harder problem. More generally, learning to control (or maximize reward) in an unknown system is a foundational problem in machine learning and control theory. One widely-used approach is to learn the dynamical matrices from data, after which one can simply apply the Kalman filter. Unfortunately, this approach runs into computational barriers: the usual formulation of this problem is nonconvex. System identification techniques provide various practical algorithms for this problem~\cite{ljung1998system}. However, these algorithms, such as EM~\cite{roweis1999unifying}, lack rigorous end-to-end guarantees, and are often unstable or find suboptimal solutions in high dimensions.

In this work, we bypass the state-space representation of an LDS, and analyze the statistical guarantees of learning an autoregressive filter directly. This allows us to compete with the predictions of the steady-state Kalman filter, without the computationally intractable task of explicitly identifying the system. We present a polynomial-time algorithm for learning an autoregressive filter for time-series prediction. The predictor has robust ($\cal H_\iy$) learning guarantees, which do not seem to be attained by naive least-squares. 

\subsection{Background}

Our primary motivation is the following question: \emph{can we learn the Kalman filter directly, without learning the system?} 
We consider the setting of a linear dynamical system with hidden state, defined by 
\begin{align}\label{eq:lds1}
h(t) &= A h(t-1) + B x(t-1)+ \xi(t) \\
y(t) & = C h(t) + \eta(t) ,
\label{eq:lds2}
\end{align}
where $x(t)\in \R^m$ are inputs, $h(t)\in \R^d$ are hidden states, $y(t)\in \R^n$ are outputs, 
$A\in \R^{d\times d}$, $B\in \R^{d\times m}$, $C\in \R^{n\times d}$, and $\xi(t)\in \R^d$ and $\eta(t)\in \R^n$ are independent zero-mean noise (assumed Gaussian to use the Kalman filter). Crucially, only the $y(t)$, and not the $h(t)$, are observed.
A classic approach to learning the dynamics from data is \emph{subspace identification} \cite{ho1966effective,van2012subspace}, for which statistical guarantees only exist in the asymptotic regime or under stringent assumptions. In the presence of noise, these methods are often used to initialize the EM algorithm \cite{roweis1999unifying}, a classic heuristic for a non-convex objective.

Another age-old model for dynamical systems is the autoregressive-moving average (ARMA) model \citep{Hamilton94,BoxJenRe94,BroDav09}, which models latent perturbations using a \emph{moving average} process. A central technique here is to recover an ARMA model by solving the Yule-Walker equations. However, to our knowledge, existing work on provably learning these models is limited to asymptotic guarantees.

\subsection{Our results}

We show that under certain stability conditions of the Kalman filter, 
we can bypass proper identification of the system, and still converge to the performance of the Kalman filter. 
We take an improper learning approach, reducing this problem to the general problem of \emph{learning an autoregressive model}. 

Our algorithm is based off a simple and familiar algorithm in time series analysis: using a sine-wave input design to fit an autoregressive model using least-squares. 
However, a key problem with the ordinary least-squares approach is that it does not provide learning guarantees under worst-case input (that we have not necessarily seen), i.e., in the $\cal H_\iy$ norm. 
Such worst-case bounds are important because in the usual control-theoretic framework, bounds under the $\cal H_\iy$ norm are used to obtain guarantees for robust control. 

To obtain $\cal H_\iy$ bounds for learning an autoregressive model, we augment our algorithm with a $L^\infty$ objective to learn a predictor that is robust in the $\cal H_\iy$ sense. When applied to the Kalman filter, our work gives (to our knowledge) the first non-asymptotic sample complexity guarantees for learning an optimal autoregressive filter for estimation in a LDS. 

\subsection{Related work}

\boldorpara{LDS without hidden state, and FIR's.}
The problem of learning unknown dynamical systems has attracted a lot of recent attention from the machine learning community, due to connections to reinforcement learning and recurrent neural networks. 
Much progress has been made on the simpler related problem of learning and control in a linear dynamical model with no hidden state. Such a model is defined by
\begin{align}
h(t) &= Ah(t-1)+Bx(t-1)+\xi(t),
\end{align}
where $A$, $B$, $x(t)$, $\xi(t)$ are as before, but $h(t)$ is now observed. 
\cite{dean2017sample} consider the linear quadratic regulator (LQR)---the control problem for such a LDS---and prove that the least-squares estimator of the dynamics, given independent rollouts, is sample-efficient for this setting. 
\cite{simchowitz2018learning} show that access to independent rollouts is unnecessary; the LDS can be identified with a single rollout, even when the system is only marginally stable.

An alternative approach to identifying $A$ and $B$ is to learn the system as a finite-impluse response (FIR) filter. This is because the problem of learning a FIR filter can be thought of as a relaxation of the problem of learning a LDS, by ``unrolling'' the LDS. \cite{tu2017non} use ordinary least-squares with design inputs to learn a FIR, and give near-optimal sample complexity bounds. \cite{boczar2018finite} complete the ``identify-then-control'' pipeline by studying robust control for this estimated FIR filter. 

However, because the predictions given by a FIR filter depend only on the inputs $x(t)$, and not the outputs $y(t)$, these methods do not suffice when the system has a hidden state. Such filters can only capture the dynamics of stable systems: for unstable or marginally stable systems, the infinite impulse response filter does not decay, so it is not approximated by a short truncation. Moreover, in these works, prediction performance guarantees are given under observation noise, and become very poor under process noise; indeed, to achieve optimal filtering (as in Kalman filtering), one must regress on the output. (See Appendix~\ref{a:fir} for a simple example.) 
Our approach fills a gap in the literature, by giving a statistical analyses similar to~\cite {tu2017non} for autoregressive models.

\boldorpara{LDS with hidden state, and autoregressive models.}
When the system has a hidden state, several recent works analyze settings in which the dynamics can be identified. \cite{hardt2016gradient} show that under certain conditions on the characteristic polynomial of the system's transition matrix, gradient descent learns the parameters of a single-input single-output LDS. However, they only consider the setting of observation noise, and not process noise (i.e. $\xi(t) = 0$). 
In work concurrent to ours,~\cite{simchowitz2019learning}, building on~\cite{oymak2018non}, consider the problem of learning an autoregressive filter, and for the case of a LDS, are able to recover matrices $\ol A$, $\ol B$, $\ol C$ which give an \emph{equivalent realization} of the LDS. Although they allow for semi-parametric noise and marginally stable systems, their guarantees are for estimating the filter in operator norm, rather than the system in the more stringent $\cal H_\iy$ norm.


\cite{AnavaHMS13} show that in the online learning (regret minimization) setting, it is possible to learn an ARMA model sample-efficiently, even in the presence of adversarial (as opposed to i.i.d. Gaussian) noise. However, the regret framework is different than what is required for control, as it ensures performance only on the data that is seen; the predictor is not required to perform well on worst-case input. Furthermore, the constraint on the $\ell_1$ norm of the MA coefficients $(\beta_i)$, which they require for the dynamical stability of their estimator of residuals, is very stringent.

Finally, we note the approach of online spectral filtering for prediction in symmetric and asymmetric LDS's \cite{HSZ17,hazan2018spectral}. 
In these works, the process noise is only handled up to a multiplicative factor of the optimal filter with knowledge of the system. Intuitively, this ``competitive ratio bound'' arises because these works consider regressing only on one or a few past observations $y_t$ (in a somewhat rigid manner), rather than having the freedom to imitate an optimal autoregressive filter.

\section{Problem setting and preliminaries}

We first state the general problem of learning an autoregressive model, and then in Section~\ref{ss:kf} describe the connection to linear dynamical systems. In Section~\ref{ss:ct} we introduce some concepts from control theory and use it to write error bounds in terms of control-theoretic norms of filters (Lemma~\ref{l:bd-err}).

\subsection{Problem statement}
\label{s:prob}
A (single-input, single-output) \emph{dynamical system} converts input signals $ x(0), \ldots, x(T-1)  \in \R $ into output signals (random variables) $y(1), \ldots , y(T) \in \R $. We will assume that the data are generated by an autoregressive model:
\begin{align}\label{e:ar-sys}
y(t+1) &= g^**x(t)+h^**y(t) + \eta(t+1) = \sumz k{\iy} g^*(k) x(t-k) + \sumz k{\iy} h^*(k)y(t-k) + \eta(t+1),
\end{align}
where $\eta(t) \sim N(0,\si^2)$ is a time series of i.i.d. Gaussian noise, $g$, $h$, are supported on $\N_0$, and $x(t)=0$ for $t<0$ and $y(t)=0$ for $t\le 0$.
\begin{prb}\label{p:main}
Let $g^*,h^*\in \R^{\N_0}$ be filters. 
The learner is given black-box access to the system $\cL$ which takes inputs $x\in \R^{\N_0}$ to outputs $y\in \R^\N$ by \eqref{e:ar-sys}. 
During each rollout, the learner specifies an input design $\{ x(0), \ldots, x(T-1) \}$, and receives the corresponding output sequence. After collecting outputs from $s$ rollouts, the learner returns filters $g,h$, which specify a map from input to output signals via~\eqref{e:ar-sys}.

For an estimate $g,h$ of $g^*,h^*$, define the error in the prediction (compared to the expected value of $y(t+1)$) to be
\begin{align}\label{e:y-err}
y_{\text{err}}(t+1) &= (g-g^*)*x(t) + (h-h^*)*y(t).
\end{align}
The goal is to learn $g,h$ such that the expected error in the prediction is a small fraction $\ep_1$ of the input, plus a small fraction $\ep_2$ of the elapsed time:
\begin{align}\label{e:eps}
\E\ba{\sumo tT \ve{y_{\text{err}}(t)}^2} &\le \ep_1 \sumo tT \ve{x(t)}^2 + \ep_2 T.
\end{align}
\end{prb}

\subsection{Connection to the Kalman filter}
\label{ss:kf}

Our work is motivated by optimal state estimation in LDS's with hidden state 
given by the dynamics~\eqref{eq:lds1}--\eqref{eq:lds2}.
The Kalman filter gives the optimal solution in the case that the parameters of the LDS are  known and $h(0)$ is drawn from a gaussian with known mean $h^-(0)$ and covariance.
We can compute matrices $A_{KF}^{(t)}$, $B_{KF}^{(t)}$, and $C_{KF}^{(t)}$ such that the optimal estimate of the latent state $\wh h(t)$ and the observation $\wh y(t)$ are given by a time-varying LDS (taking the $y(t)$ as feedback) with those matrices: 
\begin{align}
h^-(t) &=A_{KF}^{(t)}h^-(t-1)  +  B_{KF}^{(t)} \coltwo{x(t)}{y(t)}\\
\wh y(t) &= C_{KF}^{(t)}h^-(t) .
\end{align}
Taking $t\to \iy$, under mild non-degeneracy conditions the covariance of the latent state conditioned on the observations approaches a fixed covariance matrix $\Si_h$, and the matrices $A_{KF}^{(t)}$, $B_{KF}^{(t)}$, and $C_{KF}^{(t)}$ approach certain fixed matrices $A_{KF}$, $B_{KF}$, and $C_{KF}$.  Our goal is to learn this \emph{steady-state Kalman filter} without knowing parameters of the original LDS.
\footnote{Note that if the parameters of the LDS are unknown, then any $A$, $B$, $C$ for which the law of the $y_t$ in~\eqref{eq:lds1}--\eqref{eq:lds2} is the same as the law of the actual $y_t$ is an equivalent realization. Then the Kalman filters computed from these $A$, $B$, $C$ will all give equivalent predictions, so we need not distinguish between them.} At steady-state, given $\cal F_{t-1}$ the observations up to time $t-1$, the actual hidden state $h(t)$ and observation $y(t)$ will be distributed as $h(t)|\cal F_{t-1} \sim N(h^-(t), \Si_h)$ and $y(t)|\cal F_{t-1} \sim N(\wh y(t),\Si_y)$ for some covariance matrices $\Si_h$, $\Si_y$.

Denote $B_{KF} = (B_{KF,x}\;B_{KF,y})$, where $B_{KF,x}$ and $B_{KF,y}$ are the submatrices acting on $x(t)$ and $y(t)$, respectively. Consider for simplicity the case where the input and output dimensions are 1: if the hidden state has dimension $d$, then $A_{KF}\in \R^{d\times d}$, $B_{KF,x}, B_{KF,y}\in \R^{d\times 1}$, $C_{KF}\in \R^{1\times d}$, and we simply have $\Si_h=\si_h^2$ for some $\si_h$. 
We can then ``unfold'' the Kalman filter into an equivalent autoregressive model~\eqref{e:ar-sys} by letting $g^*(t) = C_{KF}A_{KF}^tB_{KF,x}$ and $h^*(t) = C_{KF}A_{KF}^tB_{KF,y}$, and $\eta(t)\sim N(0,\si_h^2)$.\footnote{Note this is not to be confused with the $\eta(t)$ in~\eqref{eq:lds1}--\eqref{eq:lds2}: this $\eta(t)$ has larger variance because it also incorporates the uncertainty about the hidden state.} Note that the autoregressive model captures the law of the random process defined by the LDS (under what is observable at each time step, i.e., the filtration $\cal F_t$), without utilizing a hidden state.

In this setting, we again attempt to minimize the error between the prediction and the expected value when the dynamics are known, $
y_{\text{err}}(t) = \wh y(t) - \E[y(t)|\cal F_{t-1}].
$

\subsection{Preliminaries on control theory}
\label{ss:ct}
An impulse response function can be equivalently be represented as a power series.

\begin{df}
For a sequence $f\in \R^{\Z}$ define the transfer function of $f$ by $F(z)= \sum_{k\in \Z} f(k) z^{-k}$. We will always denote the transfer function of a sequence in $\R^\Z$ by the corresponding capital letter.
\end{df}
Note that if $y=f*x$, then as formal power series, $Y=FX$, and equality holds as functions for $z$ such that  $F(z)$, $X(z)$ converge absolutely. 
Translation corresponds to multiplication: the transfer function of $t\mapsto y(t+1)$ is $zY(z)$. Hence, letting $N$ be the transfer function of $\eta$, we have from~\eqref{e:ar-sys} that
\begin{align}
zY &= G^*X+H^*Y+zN\\
\implies (1-z^{-1}H^*)Y &= z^{-1}G^*X + N\\
Y&=z^{-1} G^*H^*_{\textrm{unr}}X + H^*_{\textrm{unr}}N\\
\text{where }H^*_{\textrm{unr}}(z):&=\rc{1-z^{-1}H^*(z)}.
\end{align}
Thus, we can rewrite~\eqref{e:ar-sys} as
\begin{align}\label{e:ar-unr}
y(t+1) &= \hunr^* * g^* * x(t) + \hunr^**\eta(t+1),
\end{align}
where $\hunr^*(k)$, the ``unrolled'' filter, is such that $\sumz k\iy \hunr^*(k)z^{-k} = H^*_{\textrm{unr}}(z)$. 

\begin{df}
The $\cal H_\iy$-norm of a filter is the $L^\iy$-norm of the transfer function over the unit circle $\ve{z}_2=1$:
\begin{align}
\ve{f}_{\cH_\iy} &=\ve{F}_\iy := \max_{\ve{z}_2=1}F(z).
\end{align}
The $\cal H_2$-norm of a filter is the $L^2$-norm of the transfer function over the unit circle:
\begin{align}
\ve{f}_{\cH_2} &=\ve{F}_2:=\pa{\rc{2\pi}\int_{|z|=1} |F(z)|^2\,dz}^{\rc 2}.
\end{align}
\end{df}
For the rest of the paper we will assume the system is stable, i.e., $\ve{H^*}_\iy< 1$, so that $\ve{H^*_{\textrm{unr}}}<\iy$.\footnote{The $\ve{H^*}_{\infty} < 1$ condition is necessary to do estimation of a general autoregressive filter in $\cH_{\infty}$-norm. Otherwise, it is impossible to worst-case estimation over an infinite time horizon, with only access to a finite rollout, as an input with infinite response can have arbitrarily small response over a finite horizon. This suggests that to solve the control problem over infinite time horizon of a non-stable system, one should look for weaker assumptions than learning in $\cH_{\infty}$-error that still allow control.}

The $\cH_2$-norm represents the steady state variance under iid Gaussian noise as input, and the $\cH_\iy$-norm represents the maximum norm of the output when $\ve{x}_2=1$:
\begin{align}\label{e:a1}
\ve{f}_{\cH_2}^2 &= \E_{\forall s,\eta(s)\sim N(0,1)} \ab{(f*\eta)(t)}^2\\
\label{e:a2}
\ve{f}_{\cH_\iy} &=\sup_{\ve{x}_2=1} \ve{f*x}_2.
\end{align}
From~\eqref{e:y-err} and~\eqref{e:ar-unr}, 
\begin{align}
y_{\text{err}}(t+1)&=(g-g^*)*x(t) + (h-h^*)*(\hunr^* * g^* * x)(t-1) + (h-h^*) * (\hunr^* * \eta)(t)\\
&= [(g-g^*) +\de_1* (h-h^*)* \hunr^* * g^*] * x(t) + [(h-h^*)*\hunr^*] *\eta(t)
\end{align}
where $\de_i(j) = \one_{i=j}$. 
Because $\eta$ has mean 0,
\begin{align*}
&\E_{\eta}\ba{ \sumo tT \ve{y_{\text{err}}(t)}^2}\\
&= \E_\eta \ba{\sumo tT \ve{[(g-g^*) +\de_1* (h-h^*)* \hunr^* * g^*] * x(t)}^2}
+ \E_\eta \ba{\sumo tT \ve{[(h-h^*)*\hunr^*] * \eta(t)}_2^2}.
\end{align*}
Hence from~\eqref{e:a1} and~\eqref{e:a2} we obtain the following, noting that the noise in Problem~\ref{p:main} is $N(0,\si^2)$.
\begin{lem}\label{l:bd-err}
Suppose that $\ve{H^*}_{\iy}<1$. Then in the setting of Problem~\ref{p:main},
\begin{align*}
\E\ba{\sumo tT \ve{y_{\text{err}}(t)}^2} &\le 
\ve{(G-G^*) + z^{-1}(H-H^*) H^*_{\textrm{unr}} G^*}_{\iy}^2
\ve{x}^2+ 
\ve{(H-H^*)H^*_{\textrm{unr}}}_2^2\si^2T
\end{align*}
\end{lem}

We will approximate $g^*,h^*$ with finite-length filters of length $r$, so we need to make sure $r$ is large enough to capture most of the response. For this, we use the following definition and lemma from~\cite{tu2017non} which gives a sufficient length in terms of the desired error and a $\cal H_\iy$ norm. 
\begin{df}[{Sufficient length condition, \cite[Definition 1]{tu2017non}}]
We say that a Laurent series $F$ has stability radius $\rh\in (0,1)$ if $F$ converges for $\set{x\in \C}{|x|>\rh}$. 
Let $F$ be stable with stability radius $\rh\in (0,1)$. Fix $\ep>0$. Define
\begin{align}
R(\ep) &=\ce{
\inf_{\rh<\ga<1}\rc{1-\ga}\ln \pf{\ve{F(\ga z)}_\iy}{\ep(1-\ga)}
}.
\end{align}
\end{df}
Note that this ``sufficient length condition'' is analogous to having a $\frac{1}{1-\rh(A)}$ dependence on the spectral norm of $A$, for learning a LDS. Indeed, a filter corresponding to a LDS will have stability radius $\rh(A)$. 
\begin{lem}[{\cite[Lemma 4.1]{tu2017non}}]\label{l:suff-l}
Suppose $F$ is stable with stability radius $\rh\in (0,1)$. Then 
$\ve{f_{\ge L}}_1 := \sum_{k\ge L} |f(k)| \le \max_{\rh<\ga<1} \fc{\ve{F(\ga z)}_\iy\ga^L}{1-\ga}.$
Hence, if $L\ge R(\ep)$, then $
\ve{f_{\ge L}}_1 \le \ep$. 
\end{lem}

\section{Algorithm and main theorem}


We motivate our main algorithm, Algorithm~\ref{a:ls}. The most natural algorithm is the following: let the inputs be sinusoids at equally spaced frequencies, and solve a least-squares problem for $g,h$. 
However, ordinary least-squares will only give $g,h$ for which the estimation error is small for \emph{random} input, while we desire $g,h$ for which the estimation error is small for \emph{worst-case} input; in other words, 
it gives an average-case ($\cH_2$), rather than the worst-case ($\cH_\iy$) bound that we desire. This is analogous to the difference between estimating a $r\times r$ matrix in Frobenius and operator norm; the Frobenius norm trivially bounds the operator norm, but the resulting bound is typically $\sqrt r$ from optimal. 
Hence, the sample complexity bound from ordinary least-squares does not have optimal dependence on $r$.
Note that~\cite{boczar2018finite} solve the analogous problem for a FIR filter $f^*$ with least-squares without suffering an extra $\sqrt r$ factor, because in that setting, the matrix $M$ in the least-squares problem is a fixed matrix depending on the inputs, the error $f-f^*$ in the estimate is gaussian, and 
supremum bounds for gaussians are applicable. Our setting is more challenging because the $M$'s depend on noise in $y$ that we have no control over.

The first step of our algorithm is still to solve a least-squares problem. We do this in two parts: first, solve for $h_{LS}$ by regressing on zero input, and then using $h_{LS}$, solve for $g_{LS}^{(j)}$ separately for each frequency $j$. We do this to avoid the error in $h_{LS}$---larger by a factor $\sqrt r$ because it is $r$-dimensional---contributing to the error in the $g_{LS}^{(j)}$. 

The final step is to combine the $g_{LS}^{(j)}$. Because the number of frequencies is larger than the length $r$ of the filter (necessary to be able to interpolate to unseen frequencies), we cannot find a single $g$ that matches each $g_{LS}^{(j)}$ on the $j$th frequency. Keeping in mind our $\cH_\iy$ objective, we hence optimize a $L^\iy$ problem over the frequencies to interpolate the $g_{LS}^{(j)}$.

Note that in the algorithm we can just take just $0<j<\fc{cr}2$ for the $\sin$ signals because the signals for $j=0,\fc{cr}2$ are trivial; we consider $0\le j\le \fc{cr}2$ to make the notation in the proof simpler. For convenience of notation we re-index the time series to start at $t=-L$.

\begin{algorithm}[h!]
\begin{algorithmic}[1]
\STATE INPUT: $L,\ell, r, c> 4\pi$.
\STATE Collect length $T=cr$ rollouts of the $\sim 2c\ell r$ input signals starting at $t=-L$,
\begin{align}
x^{(\bullet, k)} = x^{(\bullet)} &\equiv 0 , & && 1\le k &\le c\ell r\\
x^{(j,k)}_{\cos}(t) = x^{(j)}_{\cos}(t)& = \cos\pf{2\pi jt}{cr}& 0\le j &\le \fc{cr}2 & 1\le k&\le \ell,\\
x^{(j,k)}_{\sin}(t) = x^{(j)}_{\sin}(t) & = \sin\pf{2\pi jt}{cr} & 0\le j &\le \fc{cr}2 & 1\le k&\le \ell.
\end{align}
Let the outputs be $y^{(\bullet, k)}$, $y_{\cos}^{(j,k)}$, and $y_{\sin}^{(j,k)}$. Let $M^{(\bullet, k)}\in \R^{r\times T}$, $M^{(j,k)}_{\cos,t}\in \R^{2r\times T}$, and $M^{(j,k)}_{\sin, t}\in \R^{2r\times T}$ be the matrices with columns (for $1\le t\le T$)
\begin{align}
M^{(\bullet, k)}_t &= 
y^{(\bullet, k)}(t-1:t-r)
&
M^{(j, k)}_{\cos,t} &= \coltwo{x^{(j)}_{\cos}(t-1:t-r)}{y^{(j, k)}_{\cos}(t-1:t-r)}&
M^{(j, k)}_{\sin,t} &= \coltwo{x^{(j)}_{\sin}(t-1:t-r)}{y^{(j, k)}_{\sin}(t-1:t-r)}
\end{align}
where $x(t-1:t-r)$ denotes $(x(t-1),\ldots, x(t-r))^\top$. 
\STATE Solve the following least-squares problem under zero noise. Here, $y^{(\bullet, k)}$ refers to the vector $y^{(\bullet, k)}(1:T)$.
\begin{align}\label{e:min-max-h}
h_{LS} &= \amin_h \sumo k{c\ell r}\ve{M^{(\bullet,k)\top}h - y^{(\bullet,k)}}^2.
\end{align}
\STATE Solve the following least-squares problems, for $0\le j\le \fc{cr}2$:
\begin{align}\label{e:min-max-g}
g^{(j)}_{LS}
&= \amin_{g} \sumo k{\ell}  \ba{\ve{M^{(j, k)\top}_{\cos}\coltwo g{h_{LS}} - y^{(j, k)}_{\cos}}^2
+ \ve{M^{(j, k)\top}_{\sin}\coltwo g{h_{LS}} - y^{(j, k)}_{\sin}}^2}
\end{align}
\STATE Solve and return
\begin{align}
\label{e:ls}
\coltwo{g}{h} = \amin_{g,h}\max &
\Bigg \{\rc{r} \sumo k{c\ell r} \ve{M^{(\bullet, k)\top}\pa{h - h_{LS}}}^2,\\
&\max_j  \sumo k{\ell}
\ba{ \ve{M^{(j, k)}_{\cos}\ba{\coltwo gh -  \coltwo{g^{(j)}_{LS}}{h_{LS}}}}^2
+ \ve{M^{(j, k)}_{\sin}\ba{\coltwo gh -  \coltwo{g^{(j)}_{LS}}{h_{LS}}}}^2}
\Bigg\}.
\nonumber
\end{align}
\end{algorithmic}
\caption{Learning an autoregressive model}
\label{a:ls}
\end{algorithm}


\begin{thm}[Learning an autoregressive model]\label{t:main-trunc}
There is $C,C'$ such that the following holds.
In the setting of Problem~\ref{p:main}, 
suppose that $\ve{G^*}_{\iy}<\iy$, $\ve{H^*}_{\iy}<1$, and 
Algorithm~\ref{a:ls} is run with $c\ge 8\pi$, burn-in time 
$L \ge \max\bc{
R_{H_{\textrm{unr}}^*}\pa{\fc{\de}{4KT\sqrt{c\ell r}}},
R_{H_{\textrm{unr}}^*G^*}
\pa{\fc{\de}{4K\sqrt{c\ell rT}}}}
$ where $K=\pa{1+\sumz t{T-2}|h^*(t)|}^2$, rollout length $T$, and $\ell \ge C^{\prime2}(r+\ln \prc{\de})$ rollouts  of each input.
Let 
\begin{align}
\ep_1:&=\fc{C}{\sqrt{c\ell T}}\pa{\ln \pf{c\ell rT}{\de}}^{\fc 32}(1+\ve{H^*}_\iy)\ve{H^*_{\textrm{unr}}}_\iy,&
\ep_2:&=\fc{C}{\sqrt{c\ell T}}\pa{\ln \pf{c\ell rT}{\de}}^{2}.
\end{align}
Then with probability $1-\de$, the algorithm returns $g,h$ such that 
\begin{align}\label{e:main}
\E\ba{\sumo tT \ve{y_{\text{err}}(t)}^2} 
&\le \ep_1^2 \ve{x}_2^2 + \ep_2^2T.
\end{align}
\end{thm}
To prove the theorem, we establish the bounds
\begin{align}
\ve{(G-G^*) + z^{-1}(H-H^*)H^*_{\textrm{unr}}G^*}_\iy &\le \ep_1
&
\ve{(H-H^*)H^*_{\textrm{unr}}}_2 &\le \si^{-1}\ep_2
\end{align}
and use Lemma~\ref{l:bd-err}. Note there is no dependence on $\si$ in~\eqref{e:main} for the following reason: smaller $\si$ means worse estimation of $\ve{(H-H^*)H^*_{\textrm{unr}}}_2$ (the response to $N(0,1)$ noise) by a factor of $\si^{-1}$, but when tested on rollouts with noise $N(0,\si^2)$, the error is not affected.

We expect the $O\prc{\sqrt{\ell T}}$ dependence on $\ell,T,r$ to be optimal: there are $O(r)$ parameters, and we have access to $O(\ell Tr)$ samples (including samples in the same rollout). We also conjucture that the $\ve{H_{\textrm{unr}}^*}_\iy$ dependence is unavoidable.

As an immediate corollary, we obtain a theorem for learning the Kalman filter. For simplicity, we state the result when $h(0)$ has the steady-state distribution, to avoid burn-in time arguments.
\begin{cor}[Improperly learning the Kalman filter]\label{t:main-kf}
Consider the system~\eqref{eq:lds1}--\eqref{eq:lds2}. Let $A_{KF}$, $B_{KF,x}$, $B_{KF,y}$, $C_{KF}$ be the Kalman filter matrices and $\si_y^2$ be the variance in the estimate of $y$, as defined in Section~\ref{ss:kf}. 
Let $G^*(z)=\sumz t{\iy} C_{KF}A_{KF}^t B_{KF,x}z^{-t}$ and $H^*(z) = \sumz t{\iy} C_{KF}A_{KF}^t B_{KF,y}z^{-t}$. Suppose that $A_{KF}$ has spectral radius $<1$, and suppose the rollouts are started with $h(0)\sim N(0, \si_h^2)$. 
Algorithm~\ref{a:ls} with parameters given in Theorem~\ref{t:main-trunc} 
returns predictions such that 
\begin{align}
\E\ba{\sumo tT \ve{y_{\text{err}}(t)}^2} 
&\le \ep_1^2 \ve{x}_2^2 + \ep_2^2T.
\end{align}
\end{cor}

\section{Proof sketch}
\label{s:proof-sketch}

It will be convenient to first prove the theorem in the case when the burn-in time is infinite. Note that by the stability assumption on $H^*$, for signals with finite $\ve{x}_\iy$, the outputs will not diverge.
\begin{thm}\label{t:main}
Theorem~\ref{t:main-trunc} holds in the setting when the burn-in time $L$ is infinite.
\end{thm}

We break the proof of Theorem~\ref{t:main-trunc} into 4 parts. The first 3 parts will prove Theorem~\ref{t:main}. \iftoggle{conference}{The full proof is in the supplement.}{The full proof is in Section~\ref{s:proof}.}

\boldorpara{Step 1 (Concentration):} 
If $\ol y=M^\top x$ and $y=\ol y+\eta$, then the error from the least-squares problem $\amin_x\ve{M^\top x-y}^2$ is $(MM^\top)^{-1}M\eta$. A simple way to bound this is to bound $MM^{\top}$ from below and $M\eta$ from above. When we have $s$ samples, and $x\in \R^r$, we expect $\ve{(MM^\top)^{-1}} \le O\prc{s}$ and $\ve{M\eta}\le O\pa{\sqrt{rs}}$.

We show that the matrices $Q^{(\bullet)}:=\sumo k{c\ell r}
 M^{(\bullet,k)} M^{(\bullet,k)\top}$ and $Q^{(j)}:=\sumo k{\ell} (M^{(j,k)}_{\cos}M^{(j,k)\top}_{\cos} + M^{(j,k)}_{\sin}M^{(j,k)\top}_{\sin})$
in the least-squares problem~\eqref{e:min-max-h} and~\eqref{e:min-max-g} concentrate using matrix concentration bounds (Lemma~\ref{l:mat-conc}), and that the terms such as $\sumo k{c\ell r} M^{(\bullet,k)}\eta^{(\bullet, k)}$
concentrate by martingale concentration (Lemma~\ref{l:vec-conc}). The main complication is to track how the error $h_{LS}-h^*$ propagates into $g^{(j)}_{LS}-g^*$ (see~\eqref{e:proj-g-h} and following computations).

\boldorpara{Step 2 (Generalization):} The bounds we obtain on $\coltwo{g^{(j)}_{LS}}{h_{LS}} - \coltwo{g^*}{h^*}$ in the direction of the $j$th frequency~\eqref{e:jgh} show that the actual solution $(g^*,h^*)$ does well in the min-max problem~\eqref{e:ls}. The solution $(g,h)$ to~\eqref{e:ls} will only do better. By concentration, the matrices in the least-squares problem $Q^{(\bullet)}$, $Q^{(j)}$ and in the actual expected square loss are comparable. Because $(g,h)$ does well in the min-max problem, it will do comparably well with respect to the actual expected loss, when the input is one of the frequencies that has been tested, $\fc{2\pi j}{cr}$. 

In this step, we already have enough to bound $\ep_2$, the error in estimation with pure noise and no input signal.

\boldorpara{Step 3 (Interpolation):} 
We've produced $(g,h)$ that is close to the actual $(g^*,h^*)$ when tested on each of the frequencies $\fc{2\pi k}{cr}$, but need to extend this bound to all frequencies. Considering transfer functions and clearing denominators, this reduces to a problem about polynomial interpolation. We use a theorem from approximation theory (Theorem~\ref{l:interp}) that bounds the maximum of a polynomial $p$ on the unit circle, given its value at $\ge \deg p$ equispaced points. Note that it is crucial here that the number of parameters in $g,h$ is less than the number of frequencies tested. 

Note that we needed to clear $1-z^{-1}H^*$ from the denominator, so we lose a factor of $\ve{H_{\textrm{unr}}^*}_{\iy} = \ve{1-z^{-1}H^*}_{\iy}$ here. We obtain a bound on $\ep_1$, finishing the proof of Theorem~\ref{t:main}.

\boldorpara{Step 4 (Truncation):}
Finally, we show that with a large burn-in time, the distribution of $y$'s will be almost indistinguishable from the steady-state distribution, and hence the algorithm still works.
\section{Proof}
\label{s:proof}

\subsection{Concentration}

We first set up notation and make some preliminary observations. A table of notation is provided in Section~\ref{s:notation}.
Let $\mathbf X^{(j)}_{\cos}\in \R^{r\times T}$ be the matrix with columns $x^{(j,k)}_{\cos}(t-1:t-r)$, $1\le t\le T$ and likewise define $\mathbf X^{(j)}_{\sin}, \mathbf Y^{(j,k)}_{\cos}, \mathbf Y^{(j,k)}_{\sin}$, so that 
$M^{(j,k)}_* = \coltwo{\mathbf X^{(j,k)}_*}{\mathbf Y^{(j,k)}_*}$ for $*\in \{\cos,\sin\}$.
Let 
$\Ga^{(\bullet)} = \E_{\eta^{(\bullet,k)}} M^{(\bullet, k)}M^{(\bullet, k)\top}$, $\Ga^{(j)}_{*,t}=\E_{\eta^{(j,k)}_*}M^{(j,k)}_{*,t} M^{(j,k)\top}_{*,t}$, $\Ga^{(j)}_{X,*,t}=\mathbf X^{(j)}_{*,t}\mathbf X^{(j)\top}_{*,t}$ where $*\in \{\cos,\sin\}$, $\eta^{(\bullet,k)},\eta^{(j,k)}_{*}$ is the noise in the various rollouts. We will also write $\eta$ for the noise from a generic rollout (so $\eta^{(\bullet, k)}$, $\eta^{(j,k)}_*$ are independent copies of $\eta$).

Let $\Ga^{(j)} = 
\Ga^{(j)}_{\cos, t} + \Ga^{(j)}_{\sin, t}$ and $\Ga^{(j)}_X=\Ga^{(j)}_{X,\cos, t} + \Ga^{(j)}_{X,\sin,t}$.
These matrices not depend on $t$, which can be seen as follows. Consider the system response to $x^{(j)}(t)=e^{\fc{2\pi i jt}{cr}}$. (Although we cannot put in complex values in the system, there is a well-defined response for complex inputs.) Let $M^{(j)}$ be the matrix with columns $M^{(j)}_t=\coltwo{x^{(j)}(t-1:t-r)}{y^{(j)}(t-1:t-r)}$, where the $y^{(j)}$ is defined as in~\eqref{e:ar-sys} except with noise equal to $\eta^{(j)}(t)=\eta^{(j)}_{\cos}(t)+i\eta^{(j)}_{\sin}(t)$, $\eta^{(j)}_{\cos}(t), \eta^{(j)}_{\sin}(t)\sim N(0,\si^2)$. Because $M^{(j)}_{t+s}$ has the same distribution as $e^{\fc{2\pi i s}{cr}}M^{(j)}_t$, the expression $\E [M^{(j)}_tM^{(j)\dagger}_t + M^{(-j)}_tM^{(-j)\dagger}_t]$ does not depend on $t$. Expanding, it equals $\rc 2\E[(M^{(j)}_{\cos, t}+iM^{(j)}_{\sin, t})(M^{(j)}_{\cos, t}-iM^{(j)}_{\sin, t})^\top + 
(M^{(j)}_{\cos, t}-iM^{(j)}_{\sin, t})(M^{(j)}_{\cos, t}+iM^{(j)}_{\sin, t})^\top] = \Ga^{(j)}_{\cos,t}+\Ga^{(j)}_{\sin,t}$. Similarly, $\Ga^{(j)}_X$ is well-defined. Note that $\Ga^{(j)}_X = \mathbf X^{(j)}_{\cos,t} \mathbf X^{(j)\top}_{\cos, t} + \mathbf X^{(j)}_{\sin, t} \mathbf X^{(j)\top}_{\sin, t}$ has rank $\le 2$, as the columns of $\mathbf X^{(j)}_{\cos, t}$ and $\mathbf X^{(j)}_{\sin, t}$ are spanned by $x^{(\pm j)}(r:1)$. 

Let $\ol y(t)$ denote the expected value of $y(t)$ given $y(s),x(s)$ for $s<t$: $\ol y(t+1) = g^**x(t)+h^**y(t)$. Let $\ol{\ol y}(t)$ denote the expected value of $y(t)$ given only the inputs $x(s)$ for $s<t$.

We first compute the error $h_{LS}-h^*$ and $g^{(j)}_{LS}-g^*$, and then the error in the mean response which is given by 
$\pa{x^{(j)}_{\cos}(r:1)^{\top} \; \ol{\ol y}^{(j)}_{\cos}(r:1)^\top}
\ba{\coltwo{g^{(j)}_{LS}}{h_{LS}}  - \coltwo{g^*}{h^*} }$, and the analogous expression for $\sin$. This is broken up into subexpressions that we apply concentration bounds to. 

\paragraph{Computing $h_{LS}-h^*$.}
Let \begin{align}
Q^{(\bullet)} &=
 \sumo k{c\ell r}
 M^{(\bullet,k)} M^{(\bullet,k)\top}
=
 \sumo k{c\ell r}  \sumo tT M^{(\bullet,k)}_t M^{(\bullet,k)\top}_t
 \\
Q^{(j)} &= 
\sumo k{\ell} (M^{(j,k)}_{\cos}M^{(j,k)\top}_{\cos} + M^{(j,k)}_{\sin}M^{(j,k)\top}_{\sin}) =
\sumo k\ell \sumo tT (M^{(j,k)}_{\cos,t}M^{(j,k)\top}_{\cos,t} + M^{(j,k)}_{\sin,t}M^{(j,k)\top}_{\sin,t}).
\end{align}
We calculate the least squares solution $h_{LS}$ and the error $h_{LS}-h^*$, noting that $y^{(\bullet, k)}=\ol y^{(\bullet, k)} + \eta^{(\bullet, k)}$. 
\begin{align}
h_{LS} &= Q^{(\bullet)-1}\sumo k{c\ell r} M^{(\bullet,k)} y^{(\bullet,k)}\\
h^*&= Q^{(\bullet)-1} \sumo k{c\ell r} M^{(\bullet,k)} \ol y^{(\bullet,k)}\\
h_{LS}-h^*&=Q^{(\bullet)-1} \sumo k{c\ell r} M^{(\bullet,k)} \eta^{(\bullet,k)}\\
&=\Ga^{(\bullet)-\rc 2}
\ub{(\Ga^{(\bullet)-\rc 2}Q^{(\bullet)} \Ga^{(\bullet)-\rc 2})^{-1}}{(0\bullet)}
\ub{\Ga^{(\bullet)-\rc 2}\sumo k{c\ell r} M^{(\bullet, k)}\eta^{(\bullet,k)}}{(1\bullet)}
\label{e:h-err}
\end{align}

\paragraph{Computing $g^{(j)}_{LS}$.}
The least squares solution $g^{(j)}_{LS}$ is
\begin{align}
g^{(j)}_{LS} &=
\rc{\ell T} \Ga^{(j)+}_X
\ba{
\sumo k{\ell} 
[\mathbf X^{(j)}_{\cos}(y^{(j,k)}_{\cos} - \mathbf Y^{(j,k)}_{\cos}h_{LS})+
\mathbf X^{(j)}_{\sin}(y^{(j,k)}_{\sin} - \mathbf Y^{(j,k)}_{\sin}h_{LS})]
}.\label{e:ls-g}
\end{align}
Noting that $\ol y^{(j,k)}_{\cos} = \mathbf Y^{(j,k)\top}_{\cos} h^* + \mathbf X^{(j)\top}_{\cos}g^*$, we calculate
\begin{align}
y^{(j,k)}_{\cos} - \mathbf Y^{(j,k)}_{\cos}h_{LS}&=
\eta^{(j,k)}_{\cos} + \ol y_{\cos}^{(j,k)} - \mathbf Y^{(j,k)\top}_{\cos}h^* - \mathbf Y^{(j,k)\top}_{\cos}(h_{LS}-h^*)\\
&=
\eta^{(j,k)}_{\cos} + \mathbf X^{(j)\top}_{\cos} g^* - \mathbf Y^{(j,k)\top}_{\cos}(h_{LS}-h^*).
\label{e:h-diff-calc}
\end{align}
The analogous equation for $\sin$ holds. Substituting~\eqref{e:h-diff-calc} into~\eqref{e:ls-g}, letting $P^{(j)}_X$ be the projection onto the column space of $\Ga^{(j)}_X$,  and noting $\rc{\ell T}\Ga^{(j)+}\sumo k{\ell}(\mathbf X^{(j)}_{\cos}\mathbf X^{(j)\top}_{\cos} + \mathbf X^{(j)}_{\sin}\mathbf X^{(j)\top}_{\sin})g^*=P^{(j)}_Xg^*$, we get
\begin{align}
g^{(j)}_{LS} &= P^{(j)}_X g^* + 
\rc{\ell T} \Ga^{(j)+}_X 
\ba{
\sumo k{cr}[ \mathbf X^{(j)}_{\cos} (\eta^{(j,k)}_{\cos} - \mathbf Y^{(j,k)\top}_{\cos}(h_{LS}-h^*)) + \mathbf X^{(j)}_{\sin} (\eta^{(j,k)}_{\sin} - \mathbf Y^{(j,k)\top}_{\sin}(h_{LS}-h^*))]
}
\end{align}

\paragraph{Computing}
$\coltwo{g^{(j)}_{LS}}{h_{LS}} - \coltwo{g^*}{h^*}$ (projected).
We now calculate the error in $\coltwo{g^{(j)}_{LS}}{h_{LS}}$, projected with  $P^{(j)}_X$. The projection is because we do not care about the absolute error  (which can be large), we only care about the mean error on the inputs $x^{(j)}_{\cos}$ and $x^{(j)}_{\sin}$, which are in the column space of $\Ga^{(j)}_X$. 
\begin{align}
\label{e:proj-g-h}
&
\matt{P^{(j)}_X}OO{I_r}
\ba{
\coltwo{g^{(j)}_{LS}}{h_{LS}} - \coltwo{g^*}{h^*}
}\\
\quad&=\coltwo{\rc{\ell T} \Ga^{(j)+}_X\sumo k{\ell}[ \mathbf X^{(j)}_{\cos} (\eta^{(j,k)}_{\cos} - \mathbf Y^{(j,k)\top}_{\cos}(h_{LS}-h^*)) + \mathbf X^{(j)}_{\sin} (\eta^{(j,k)}_{\sin} - \mathbf Y^{(j,k)\top}_{\sin}(h_{LS}-h^*))]}{h_{LS}-h^*}.
\label{e:proj-g-h2}
\end{align}
Let $\mathbf Y^{(j)}_{\cos}$ be the matrix with the mean responses to $x^{(j)}_{\cos}$, $\mathbf Y^{(j)}_{\cos,t} = \ol y^{(j)}_{\cos}(t-1:t-r)$,
and likewise for $\sin$. 

\paragraph{Computing}$\pa{x^{(j)}_{\cos}(r:1)^{\top} \; \ol{\ol y}^{(j)}_{\cos}(r:1)^\top}
\ba{\coltwo{g^{(j)}_{LS}}{h_{LS}}  - \coltwo{g^*}{h^*} }$.
Write $y^{(j,k)}_{\cos} = \ol{\ol y}^{(j)}_{\cos} + \zeta^{(j,k)}_{\cos}$ and 
$\mathbf Y^{(j,k)}_{\cos} = \ol{\ol{\mathbf Y}}^{(j)}_{\cos}  + \mathbf Z^{(j,k)}_{\cos}$, where $\ze^{(j,k)}_{\cos}$ is the noise term and $\mathbf Z^{(j,k)}_{\cos}$ has $\ze^{(j,k)}(t-1:t-r)$ as columns. (Note that $\eta^{(j,k)}_{\cos}$ only includes the new noise at each time step, while $\ze^{(j,k)}_{\cos}$ is the accummulated noise; $\ol y^{(j)}$ is the expected value given the previous observations, and $\ol{\ol y}^{(j)}$ is the mean given only the inputs.)
Then by~\eqref{e:proj-g-h2}, because $P^{(j)}_X x^{(j)}_{\cos}(r:1) = x^{(j)}_{\cos}(r:1)$,
\begin{align}
&
\pa{x^{(j)}_{\cos}(r:1)^{\top} \; \ol{\ol y}^{(j)}_{\cos}(r:1)^\top}
\ba{\coltwo{g^{(j)}_{LS}}{h_{LS}}  - \coltwo{g^*}{h^*} }\\
&=
x^{(j)}_{\cos}(r:1)^{\top}\rc{\ell T} \Ga^{(j)+}_X 
\ba{\sumo k{\ell} \mathbf X^{(j,k)}_{\cos}\eta^{(j,k)}_{\cos} +
 \mathbf X^{(j)}_{\sin}\eta^{(j,k)}_{\sin} }\\
 &\quad +
 \ba{x^{(j)}_{\cos}(r:1)^{\top} \rc{\ell T} \Ga^{(j)+}_X
  \ba{-\sumo k{\ell}( \mathbf X^{(j)}_{\cos}\mathbf Y^{(j,k)\top}_{\cos} +
 \mathbf X^{(j)}_{\sin}\mathbf Y^{(j,k)\top}_{\sin} )}
 +
\ol{\ol y}^{(j)}_{\cos}(r:1)^\top}
 (h_{LS}-h^*)\\
 &=
x^{(j)}_{\cos}(r:1)^{\top}
\rc{\ell T} \Ga^{(j)+}_X
\ba{\sumo k{\ell} \mathbf X^{(j,k)}_{\cos}\eta^{(j,k)}_{\cos} +
 \mathbf X^{(j)}_{\sin}\eta^{(j,k)}_{\sin} }\\
 &\quad +
 \ba{
 x^{(j)}_{\cos}(r:1)^{\top}
 \rc{\ell T} \Ga^{(j)+}_X
  \ba{-\sumo k{\ell}( \mathbf X^{(j)}_{\cos}\mathbf Z^{(j,k)\top}_{\cos} +
 \mathbf X^{(j)}_{\sin}\mathbf Z^{(j,k)\top}_{\sin} )}
}
 (h_{LS}-h^*)
\label{e:canc} \\
 &\quad \text{(see explanation below)}\nonumber\\
 &=\rc{\ell T} 
x^{(j)}_{\cos}(r:1)^{\top}
\Ga^{(j)+\rc 2}_X \Bigg[
\ub{ \Ga^{(j)+\rc 2}_X \sumo k\ell (\mathbf X^{(j)}_{\cos} \eta^{(j,k)}_{\cos} + \mathbf X^{(j)}_{\sin} \eta^{(j,k)}_{\sin})}{(1)} 
\label{e:gh-err}
\\
&\quad + 
\ub{ \Ga^{(j)+\rc 2}_X \sumo k\ell (\mathbf X^{(j)}_{\cos} \mathbf Z^{(j,k)}_{\cos} + \mathbf X^{(j)}_{\sin} \mathbf Z^{(j,k)}_{\sin})(h_{LS}-h^*)}{(2)}
 \Bigg]
 \label{e:gh-err2}
\end{align}
In~\eqref{e:canc} we used that $x^{(j)}_{\cos}(r:1)^\top  \Ga^{(j)+}_X\ba{-
\sumo k{\ell} (\mathbf X^{(j)}_{\cos}
\ol{\ol{\mathbf Y}}^{(j)}_{\cos} + 
\mathbf X^{(j)}_{\sin}
\ol{\ol{\mathbf Y}}^{(j)}_{\sin}) 
}+\ol{\ol y}^{(j)}_{\cos}(r:1)^\top=0$. To see this, let $A$ be the matrix sending $x^{(j)}_{*}(t-1:t-r)\mapsto \ol{\ol y}^{(j)}_{*}(t-1:t-r)$ for $*\in \{\cos,\sin\}$. Then this equals $x^{(j)}_{\cos}(r:1)^\top \Ga^{(j)+}_X\ba{-\sumo k{\ell}\sumo tT (\mathbf X^{(j)}_{\cos,t}\mathbf X^{(j)\top}_{\cos,t}+\mathbf X^{(j)}_{\sin,t}\mathbf X^{(j)\top}_{\sin,t})A^\top}+x^{(j)}_{\cos}(r:1)^\top A^\top=0$.

\paragraph{Computing}$\ve{\Ga^{(j)\rc 2}\ba{\coltwo{g^{(j)}_{LS}}{h_{LS}}  - \coltwo{g^*}{h^*} }}$.
\begin{align}
&\ve{\Ga^{(j)\rc 2}\ba{\coltwo{g^{(j)}_{LS}}{h_{LS}}  - \coltwo{g^*}{h^*} }}\\
& = \ve{\coltwo{\Ga^{(j)\rc 2}_{\cos,r+1}}{\Ga^{(j)\rc 2}_{\sin,r+1}}
\ba{\coltwo{g^{(j)}_{LS}}{h_{LS}}  - \coltwo{g^*}{h^*} }}\\
&\quad \text{because $\Ga^{(j)} = \Ga^{(j)}_{\cos, t}+\Ga^{(j)}_{\sin, t}$ for any $t$}\\
&=
\sqrt{\sum_{*\in \{\cos,\sin\}}\ve{
\Ga^{(j)\rc 2}_{*,r+1}
\ba{\coltwo{g^{(j)}_{LS}}{h_{LS}}  - \coltwo{g^*}{h^*} }}^2 } \\
&=
\sqrt{\sum_{*\in \{\cos,\sin\}}
\ba{\coltwo{g^{(j)}_{LS}}{h_{LS}}  - \coltwo{g^*}{h^*} }^\top
\EE_{\eta^{(j)}_{*}}
\coltwo{x^{(j)}_{*}(r:1)}{
y^{(j)}_{*}(r:1)}
\pa{x^{(j)}_{*}(r:1)^{\top}\;
 y^{(j)}_{*}(r:1)^\top}
\ba{\coltwo{g^{(j)}_{LS}}{h_{LS}}  - \coltwo{g^*}{h^*} }} 
\end{align}
Because $\ze^{(j)}_*(r:1)$ has mean 0, 
\begin{align}
&\EE_{\eta^{(j)}_{*}}
\coltwo{x^{(j)}_{*}(r:1)}{
y^{(j)}_{*}(r:1)}
\pa{x^{(j)}_{*}(r:1)^{\top}\;
 y^{(j)}_{*}(r:1)^\top}\\
&=
\EE_{\eta^{(j)}_{*}}
\ba{
	\coltwo{x^{(j)}_{*}(r:1)}{
	\ol{\ol y}^{(j)}_{*}(r:1)}
	\pa{x^{(j)}_{*}(r:1)^{\top}\;
 	\ol{\ol y}^{(j)}_{*}(r:1)^\top}
 	+
 	\coltwo{0}{
	\zeta^{(j)}_{*}(r:1)}
	\pa{0\;
 	\zeta^{(j)}_{*}(r:1)^\top}
 }
\end{align}
Hence, 
\begin{align}
&\ve{\Ga^{(j)\rc 2}\ba{\coltwo{g^{(j)}_{LS}}{h_{LS}}  - \coltwo{g^*}{h^*} }}\\
&=
\sqrt{\sum_{*\in \{\cos,\sin\}}\ve{
\pa{x^{(j)}_{*}(r:1)^{\top}\;
\ol{\ol{y}}^{(j)}_{*}(r:1)^\top}
\ba{\coltwo{g^{(j)}_{LS}}{h_{LS}}  - \coltwo{g^*}{h^*} }}^2 
+
\EE_{\eta^{(j)}_{*}} \ve{\zeta^{(j)}_{*}(r:1)^\top (h_{LS}-h^*)}^2
}.
\label{e:Ga12-deltagh}
\end{align}

\paragraph{Prospectus.} 
In Section~\ref{s:mat-conc}, we lower bound $(0\bullet)$ in~\eqref{e:h-err}, and
in Section~\ref{s:vec-conc} we upper bound 
$(1\bullet)$ in~\eqref{e:h-err} and (1) in~\eqref{e:gh-err}. In Section~\ref{s:tog} we bound (2) in~\eqref{e:gh-err2} and put the bounds together to obtain (for some $C_7$),
\begin{align}
\ve{\Ga^{(j)\rc 2}\ba{\coltwo{g^{(j)}_{LS}}{h_{LS}}  - \coltwo{g^*}{h^*} }}
&\le \fc{C_7}{\sqrt{\ell T}}\pa{\ln \pf{c\ell rT}{\de}}^2
\label{e:gagh0}
\end{align}



\subsubsection{Matrix concentration}\label{s:mat-conc}

\begin{lem}[Concentration of sample covariance]
\label{l:samp-cov}
\label{conc:samp-cov}
There are universal constants $C_1,C_2$ such that the following hold. Let $v_t\sim N(0, \Si)$ be iid, and let $\Si_m = \rc{m} \sumo tm v_tv_t^\top\in \R^{m\times m}$. Then for $\ep=C_1
\pa{\sfc{r+u}{m} + \fc{r+u}{m}}$, 
\begin{align}
\Pj\pa{(1-\ep)\Si^{\rc 2}\preceq \Si_m \preceq (1+\ep)\Si^{\rc 2}}\ge 1-2e^{-u}.
\end{align}
Moreover, when $\ep\le 1$ and $m\ge \pf{C_2}{\ep}^2 \pa{r+\ln \pf{2}{\de}}$, then 
\begin{align}
\Pj\pa{(1-\ep)\Si^{\rc 2}\preceq \Si_m \preceq (1+\ep)\Si^{\rc 2}} &\ge 1-\de.
\end{align}
\end{lem}
\begin{proof}
The first part follows from \cite[4.7.3]{vershynin2018high} on $\Si^{+\rc 2}v_t\sim N(0,P)$ where $P$ is the projection onto the column space of $\Si$. ($A^+$ denotes the pseudoinverse of $A$.)

To get the second part from the first part, note that when $m\ge r+\ln \pf{2}{\de}$, we can bound $\ep_1:=C_1\pa{\sfc{r+\ln \pf{2}{\de}}{m} + \fc{r+\ln \pf{2}{\de}}m}\le C_2\sfc{r+\ln \pf{2}{\de}}{m}$ for $C_2=2C_1$. This is $\le \ep$ under the condition on $m$. Hence
\begin{align}
\Pj\pa{
(1-\ep)\Si^{\rc 2}\preceq \Si_m \preceq (1+\ep)\Si^{\rc 2}
}
\ge \Pj\pa{
(1-\ep_1)\Si^{\rc 2}\preceq \Si_m \preceq (1+\ep_1)\Si^{\rc 2}}
\ge 1-\de.
\end{align}
\end{proof}

\begin{lem}[Bounding $(0\bullet)$ in~\eqref{e:h-err}, etc.]\label{l:mat-conc}
For $\ell \ge \pf{C_2}{\ep}^2 \pa{r + \ln \pf{2}{\de}}$, 
\begin{align}
\label{e:mat-conc1}
\Pj\Bigg(
(1-\ep)c\ell r \Ga^{(\bullet)}\preceq
{\sumo k{c\ell r} M^{(\bullet,k)}_t M^{(\bullet,k)\top}_t}\preceq (1+\ep)c\ell r \Ga^{(\bullet)}
\Bigg) &\ge 1-\de\\
\label{e:mat-conc2}
\Pj\pa{
(1-\ep)\ell \Ga^{(j)}_{\cos,t} \preceq
\sumo k{\ell} M^{(j,k)}_{\cos,t} M^{(j,k)\top}_{\cos,t}\preceq (1+\ep)\ell \Ga^{(j)}_{\cos,t}
} &\ge 1-\de\\
\label{e:mat-conc3}
\Pj\pa{
(1-\ep)\ell \Ga^{(j)}_{\sin,t} \preceq
\sumo k{\ell} M^{(j,k)}_{\sin,t} M^{(j,k)\top}_{\sin,t}\preceq (1+\ep)\ell \Ga^{(j)}_{\sin,t}
} &\ge 1-\de\\
\label{e:mat-conc4}
\Pj\Bigg(
(1-\ep)c\ell T r \Ga^{(\bullet)}\preceq
\ub{\sumo tT \sumo k{c\ell r} M^{(\bullet,k)}_t M^{(\bullet,k)\top}_t}{Q^{(\bullet)}}\preceq (1+\ep)c\ell r \Ga^{(\bullet)}
\Bigg) &\ge 1-T\de\\
\label{e:mat-conc5}
\Pj\Bigg(
(1-\ep)\ell T \Ga^{(j)} \preceq
\ub{
\sumo tT\sumo k{\ell} \ba{M^{(j,k)}_{\cos,t} M^{(j,k)\top}_{\cos,t}+M^{(j,k)}_{\sin,t} M^{(j,k)\top}_{\sin,t}}}{Q^{(j)}}\preceq (1+\ep)\ell T\Ga^{(j)}
\Bigg) &\ge 1-2T\de.
\end{align}
\end{lem}
\begin{proof}
The first three inequalities follow from applying Lemma~\ref{conc:samp-cov} to $M^{(\bullet, k)}_t\sim N(0,\Ga^{(\bullet)})$, $M^{(j,k)}_{\cos,t}\sim N(0,\Ga^{(j)}_{\cos,t})$ and $M^{(j,k)}_{\sin,t}\sim N(0,\Ga^{(j)}_{\sin,t})$. The last two inequalities follow from a union bound.
\end{proof}
Note that we used independence between rollouts to obtain concentration, and union-bound within the rollouts.

\subsubsection{Vector concentration}
\label{s:vec-conc}

We use the following two lemmas.
\begin{lem}[$\chi^2_d$-tail bound, \cite{laurent2000adaptive}]\label{l:chi-tail}
For $t\ge 0$, 
\begin{align}
\Pj_{x\sim N(0,I_d)}\pa{\ve{x}^2 \ge (d + 2(\sqrt{dt}+t))}\le e^{-t}
\end{align}
Thus letting $C(d,\de) := \pa{d + 2\pa{\sqrt{d\ln \prc{\de}}+\ln \prc{\de}}}^{\rc 2}$, $\Pj_{x\sim N(0,I_d)}(\ve{x} \ge C(d,\de))\le \de$.
\end{lem}
Note that $C(d,\de) = O\pa{\sqrt d + \sqrt{\ln \prc{\de}}}$. 


\begin{lem}[Azuma's inequality for vectors, \cite{hayes2005large}]
\label{t:azuma-vec}
\label{l:azuma-vec}
Let $X_t,t\ge 0$ be a discrete-time martingale taking values in a real Euclidean space. Suppose that $X_0=0$ and for all $n\ge 1$, $\ve{X_n-X_{n-1}}\le c$.  Then
\begin{align}
\Pj(\ve{X_n}\ge a) &\le 2 e^{1-\fc{\pa{\fc ac-1}^2}{2n}}.
\end{align}
\end{lem}
\begin{lem}[Bounding $(1\bullet)$, (1) in~\eqref{e:h-err} and~\eqref{e:gh-err}]\label{l:vec-conc}
The following hold:
\begin{align}
\Pj\pa{
	\ve{
		\sumo k{c\ell r} \Ga^{(\bullet)-\rc 2} M^{(\bullet,k)}\eta^{(\bullet, k)}
	} \ge 
	3 C\pa{r,\fc{\de}{4c\ell r T}} C\pa{1,\fc{\de}{4c\ell rT}}
 \sqrt{c\ell r T\ln \pf{4}{\de}}
}
\label{e:vec-conc1}
&\le \de\\
\Pj\pa{
\ve{\sumo k{\ell}\Ga^{(j)+\rc 2}_X [\mathbf X^{(j)}_{\cos} \eta^{(j,k)}_{\cos}+\mathbf X^{(j)}_{\sin,t} \eta^{(j,k)}_{\sin}]}\ge 
3 C\pa{1,\fc{\de}{4\ell T}}
 \sqrt{2\ell T\ln \pf{4}{\de}}
}
&\le \de.\label{e:vec-conc2}
\end{align}
\end{lem}
\begin{proof}
Consider the $c\ell rT$ partial sums of $\sumo k{c\ell r}\sumo tT \one[(A_{t,k}\cup B_{\bullet, t, k})^c]\Ga^{(\bullet)-\rc 2}M^{(\bullet, k)}_t\eta^{(\bullet, k)}(t)$ where the events are defined as
\begin{align}
A_{\bullet, t,k}&=\bc{\ve{\eta^{(\bullet, k)}(t)}>C\pa{1,\fc{\de}{4c\ell rT}}}\\
B_{\bullet, t, k}&= \bc{\ve{\Ga^{(\bullet)-\rc 2}M^{(\bullet, k)}_t}>C\pa{r,\fc{\de}{4c\ell rT}}}.
\end{align}
Note this is a martingale as $M^{(\bullet, k)}_t$ is determined by $\eta^{(\bullet, k)}(s)$ for $s<t$, so Lemma~\ref{t:azuma-vec} applies. Note that $\Ga^{(\bullet)-\rc 2}M^{(\bullet, k)}_t\sim N(0,I_r)$. We have by Lemma~\ref{l:chi-tail} that for $a=3 C\pa{r,\fc{\de}{4c\ell r T}} C\pa{1,\fc{\de}{4c\ell rT}}
 \sqrt{c\ell r T\ln \pf{4}{\de}}$, 
 $\Pj(A_{\bullet, t, k}), \Pj(B_{\bullet, t, k}) \le \fc{\de}{4c\ell rT}$. Hence 
\begin{align}
&
\Pj\pa{
\ve{
\sumo k{c\ell r}
\ba{
 \Ga^{(\bullet)-\rc 2}M^{(\bullet,k)}\eta^{(\bullet,k)}
}
}
\ge a}\\
&\le \sumo k{c\ell r}\sumo tT [\Pj(A_{\bullet, t, k})+\Pj(B_{\bullet, t, k})]+
\Pj\pa{
\ve{
\sumo k{c\ell r}\sumo tT
\ba{
\one[(A_{\bullet, t, k}\cup B_{\bullet, t, k})^c] \Ga^{(\bullet)+\rc 2}M^{(\bullet,k)}_t\eta^{(\bullet,k)}(t)
}
}
\ge a}\\
& \le 2c\ell r T \fc{\de}{4c\ell rT} + 
\fc{\de}2=\de
\end{align}
where in the last inequality, we use Lemma~\ref{l:vec-conc} and note that the definition of $a$ implies
because the following implications hold:
\begin{align}
a&\ge 
C\pa{r,\fc{\de}{4c\ell r T}} C\pa{1,\fc{\de}{4c\ell rT}} \pa{1+\sqrt{2c\ell r T\pa{1+\ln \pf{4}{\de}}}}\\
\fc\de 2&\ge
2e^{1-\fc{\pa{\fc{a}{C\pa{r,\fc{\de}{4c\ell r T}} C\pa{1,\fc{\de}{4c\ell rT}} }-1}^2}{2c\ell r T}} .
\end{align}
 
Similarly, consider the $2\ell T$ partial sums of $\sumo k{\ell} \sumo t{T}\sum_{*\in\{\cos,\sin\}}
 \one\ba{
A_{*,t,k}^c 
 } 
 \Ga^{(j)+\rc 2}_X \mathbf X^{(j)}_{*,t}\eta^{(j,k)}_{*}(t)$ where $A_{*,t,k} = \bc{\ve{\eta^{(\bullet, k)}(t)} > C(1,\fc{\de}{4rT})}$. By Lemma~\ref{t:azuma-vec} and 
Lemma~\ref{l:chi-tail}, for $a=3 C\pa{1,\fc{\de}{4\ell T}}
 \sqrt{2\ell T\ln \pf{4}{\de}}$,
\begin{align}
&
\Pj\pa{
\ve{
\sumo k{\ell} \sumo t{T}\sum_{*\in \{\cos,\sin\}}
\ba{
 \Ga^{(j)+\rc 2}_X \mathbf X^{(j,k)}_{*,t}\eta^{(j,k)}_{*}(t) 
}
}
\ge a}\\
&\le 
\sumo k{\ell}\sumo tT
\Pj(A_{*,t,k}) + 
\sumo k{\ell} \sumo t{T}\sum_{*\in\{\cos,\sin\}}
 \one\ba{
A_{*,t,k}^c 
 } 
 \Ga^{(j)+\rc 2}_X \mathbf X^{(j)}_{*,t}\eta^{(j,k)}_{*}(t)\\
&\le 2\ell T\pf{\de}{4\ell T}  + \fc{\de}2=\de.
 \label{e:vec-conc}
\end{align}\end{proof}

\subsubsection{Putting it together}
\label{s:tog}
Recall we are trying to bound $\pa{x^{(j)}_{\cos}(r:1)^{\top} \; \ol{\ol y}^{(j)}_{\cos}(r:1)^\top}
\ba{\coltwo{g^{(j)}_{LS}}{h_{LS}}  - \coltwo{g^*}{h^*} }$ by bounding~\eqref{e:gh-err}--\eqref{e:gh-err2}. 
There are constants $C_3,C_4,\ldots$ so that the following hold. 

\paragraph{Bounding (1) in \eqref{e:gh-err}.} By~\eqref{e:vec-conc2} in Lemma~\ref{l:vec-conc}, 
\begin{align}
\Pj\pa{
\ve{\sumo k{\ell}\Ga^{(j)+\rc 2}_X [\mathbf X^{(j)}_{\cos} \eta^{(j,k)}_{\cos}+\mathbf X^{(j)}_{\sin,t} \eta^{(j,k)}_{\sin}]}\ge 
3 C\pa{1,\fc{\de}{4\ell T}}
 \sqrt{2\ell T\ln \pf{4}{\de}}
} &\le \de.
\label{e:tog1}
\end{align}

\paragraph{Bounding (2) in \eqref{e:gh-err2}.}
By Lemma~\ref{l:mat-conc} with $\de \mapsfrom \fc{\de}{T}$ and $\ep\mapsfrom \rc 2$ and Lemma~\ref{l:vec-conc}, for $\ell \ge 4C_2^2\pa{r+\ln \pf{2T}{\de}}$, with probability $1-2\de$,
\begin{align}
\ve{\Ga^{(\bullet)\rc 2}(h_{LS}-h^*)}
&\le 
\ve{(\Ga^{(\bullet)-\rc 2}Q^{(\bullet)} \Ga^{(\bullet)-\rc 2})^{-1}}
\ve{\Ga^{(\bullet)-\rc 2}\sumo k{c\ell r} M^{(\bullet, k)}\eta^{(\bullet,k)}}
&\text{by \eqref{e:h-err}}
\\
&\le
\fc{2}{c\ell rT} \cdot 
3 C\pa{r,\fc{\de}{4c\ell r T}} C\pa{1,\fc{\de}{4c\ell rT}}
 \sqrt{c\ell r T\ln \pf{4}{\de}}\\
&\le \fc{C_3}{\sqrt{c\ell T}}\pa{\ln \pf{c\ell r T}{\de}}^{\fc 32}
\label{e:Ga12delta-h}
\end{align}
Note that 
$\ze^{(j,k)}_*$ is distributed the same as $y^{(\bullet, k)}$. 
Hence  
$\sumo k{\ell} \Ga^{(\bullet)-\rc 2}\zeta^{(j,k)}_{\cos}(t-1:t-r)\sim N(0,\ell I_d)$, so by Lemma~\ref{l:chi-tail}, for each $t$,
\begin{align}
\Pj\pa{
\ve{\sumo k{\ell} \Ga^{(\bullet)-\rc 2}\zeta^{(j,k)}_{\cos}(t-1:t-r)}
\ge \sqrt{\ell}C\pa{r,\fc{\de}{2T}}}
&\le \fc{\de}{2T}
\end{align}
and similarly for $\sin$. Thus, 
\begin{align}
\Pj\pa{
\ve{\sumo k{\ell}\coltwo{\mathbf Z^{(j,k)\top}_{\cos}}{\mathbf Z^{(j,k)\top}_{\sin}} \Ga^{(\bullet)-\rc 2}}
\ge
\sqrt{2T\ell}C\pa{r,\fc{\de}{2T}}
}&\le \de
\end{align}
Thus with probability $\ge 1-3\de$,
\begin{align}
\ve{\sumo k{\ell}\coltwo{\mathbf Z^{(j,k)\top}_{\cos}}{\mathbf Z^{(j,k)\top}_{\sin}}(h_{LS}-h^*)}
&=\ve{\sumo k{\ell}\coltwo{\mathbf Z^{(j,k)\top}_{\cos}}{\mathbf Z^{(j,k)\top}_{\sin}} \Ga^{(\bullet)-\rc 2}}\ve{\Ga^{(\bullet)\rc 2}(h_{LS}-h^*)}\\
&\le \sqrt{2T\ell}C\pa{r,\fc{\de}{2T}}
\fc{C_3}{\sqrt{c\ell T}}\pa{\ln \pf{c\ell r T}{\de}}^{\fc 32}\\
&\le C_4 \sfc{r}{c}\pa{\ln \pf{c\ell r T}{\de}}^2.
\end{align}
For $\Si\succeq 0$, $v^\top (vv^\top +\Si)^+v\le 1$\footnote{By the Sherman-Morrison formula, $v^\top(vv^\top +\Si)^{-1}v=v^\top \Si^{-1}v - \fc{(v^\top \Si^{-1}v)^2}{1+v^\top \Si^{-1}v}\le v^\top \Si^{-1}v$.}, so each column of $\Ga^{(j)+\rc2}(\mathbf X^{(j)}_{\cos}\; \mathbf X^{(j)}_{\sin})$ has norm $\le 1$. Then,
\begin{align}
\ve{\Ga^{(j)+\rc 2}_X(\mathbf X_{\cos}^{(j)}\; \mathbf X_{\sin}^{(j)}) 
\sumo k{\ell}\coltwo{\mathbf Z^{(j,k)\top}_{\cos}}{\mathbf Z^{(j,k)\top}_{\sin}}(h_{LS}-h^*)}&\le \fc{C_4\sqrt{2Tr}\pa{\ln \pf{c\ell r T}{\de}}^{2}}{\sqrt c}.
\label{e:tog2}
\end{align}

\paragraph{Bounding}$\pa{x^{(j)}_{\cos}(r:1)^{\top} \; \ol{\ol y}^{(j)}_{\cos}(r:1)^\top}
\ba{\coltwo{g^{(j)}_{LS}}{h_{LS}}  - \coltwo{g^*}{h^*} }$.  
Combining~\eqref{e:tog1} and~\eqref{e:tog2}, with probability $1-4\de$,
\begin{align}
\rc{\ell T}[(1)+(2)]&\le C_5 \prc{\ell T}\pa{\sqrt{\ell T} \ln \pf{\ell T}{\de} + 
\sfc{Tr}{c} \pa{\ln \pf{\ell rT}{\de}}^2
}\le \fc{C_6}{\sqrt{\ell T}}\pa{\ln \pf{c\ell rT}{\de}}^2
\end{align}
because $\ell \ge \fc{r}{c}$. 
Thus by~\eqref{e:gh-err}--\eqref{e:gh-err2},
\begin{align}
\label{e:jgh}
\ve{
\pa{x^{(j)}_{\cos}(r:1)^{\top}\;
\ol{\ol y}^{(j)}_{\cos}(r:1)^\top}
\ba{\coltwo{g^{(j)}_{LS}}{h_{LS}}  - \coltwo{g^*}{h^*} }}
&\le \ve{ x^{(j)}_{\cos}(r:1)^{\top}\Ga^{(j)+\rc 2}}
\fc{C_6}{\sqrt{\ell T}}\pa{\ln \pf{c\ell rT}{\de}}^2\\
&\le \fc{C_6}{\sqrt{\ell T}}\pa{\ln \pf{c\ell rT}{\de}}^2
\label{e:xytgh}
\end{align}
The analogous bound holds for $\sin$.

\paragraph{Bounding}$\ve{\Ga^{(j)\rc 2}\ba{\coltwo{g^{(j)}_{LS}}{h_{LS}}  - \coltwo{g^*}{h^*} }}$. 
First, note that $\ze^{(j)}_*(r:1)\sim N(0,\Ga^{(\bullet)})$ so
\begin{align}
\EE_{\eta^{(j)}_{*}} \ve{\zeta^{(j)}_{*}(r:1)^\top (h_{LS}-h^*)}^2
&\le 
\ve{\Ga^{(\bullet)\rc 2}(h_{LS}-h^*)}^2\\
&\le 
\fc{C_3}{\sqrt{c\ell T}} \pa{\ln \pf{c\ell rT}\de}^{\fc 32}
\label{e:zetah-bd}
\end{align}
provided that~\eqref{e:Ga12delta-h} holds.
Now replace $\de\mapsfrom \fc{\de}8$. By~\eqref{e:Ga12-deltagh},~\eqref{e:xytgh}, and~\eqref{e:zetah-bd}, with probability $1-\de$,
\begin{align}
\ve{\Ga^{(j)\rc 2}\ba{\coltwo{g^{(j)}_{LS}}{h_{LS}}  - \coltwo{g^*}{h^*} }}
&\le \fc{C_7}{\sqrt{\ell T}} \pa{\ln \pf{c\ell rT}{\de}}^2.
\label{e:Gagh-bd}
\end{align}

\subsection{Generalization}

We now compute the performance of $g^*,h^*$ on the minimax problem.
Let 
\begin{align}
L^{(\bullet)}(h) &= \rc{r}\sumo k{c\ell r} 
\ve{M^{(\bullet, k)\top}(h-h_{LS})}^2
\\
L^{(j)}(g,h) &= \sumo k{\ell}
\ba{ \ve{M^{(j,k)\top}_{\cos}\coltwo gh - y^{(j,k)}_{\cos}}^2
+ \ve{M^{(j,k)\top}_{\sin}\coltwo gh - y^{(j,k)}_{\sin}}^2}
.
\end{align}
Note that
\begin{align}
L^{(\bullet)}(h) - L^{(\bullet)}(h_{LS})
&=\rc r(h-h_{LS})^\top Q^{(\bullet)} (h-h_{LS})\\
L^{(j)}(g,h) - L^{(j)}(g_{LS},h_{LS}) &= \ba{\coltwo{g}{h} - \coltwo{g^{(j)}_{LS}}{h_{LS}}}^\top Q^{(j)} \ba{\coltwo{g}{h} - \coltwo{g^{(j)}_{LS}}{h_{LS}}}.
\end{align}
We have that with probability $\ge 1-\de$, by~\eqref{e:mat-conc1} in Lemma~\ref{l:mat-conc} and~\eqref{e:Ga12delta-h},
\begin{align}
L^{(\bullet)}(h^*) - L^{(j)}(h_{LS}) 
&\le \rc r\ve{\Ga^{(\bullet)-\rc 2} Q^{(\bullet)} \Ga^{(\bullet)-\rc 2}} \ve{\Ga^{(\bullet)\rc 2} (h^*-h_{LS})}^2\\
&\le 
C_8 \rc{r} (c\ell rT) \rc{c\ell T} \pa{\ln \pf{c\ell rT}{\de}}^3
=
C_8 \pa{\ln \pf{c\ell rT}{\de}}^3
\end{align}
By~\eqref{e:mat-conc5} in Lemma~\ref{l:mat-conc} and~\eqref{e:Gagh-bd},
\begin{align}
L^{(j)}(g^*,h^*) - L^{(j)}(g^{(j)}_{LS},h_{LS}) 
&\le \ve{\Ga^{(j)+\rc 2} Q^{(j)} \Ga^{(j)+\rc 2}} \ve{\Ga^{(j)\rc 2} \ba{\coltwo{g^*}{h^*} - \coltwo{g_{LS}^{(j)}}{h_{LS}}}}^2\\
&\le 
C_8 \ell T \rc{\ell T} \pa{\ln \pf{c\ell rT}{\de}}^3=
C_8\pa{\ln \pf{c\ell rT}{\de}}^3.
\end{align}
Because $\coltwo gh$ is the argmin of~\eqref{e:ls}, we have 
$\ve{Q^{(\bullet)\rc2}(h-h_{LS})}_2^2=
r[L^{(\bullet)}(h) -L^{(\bullet)}(h_{LS})]\le C_8r\pa{\ln \pf{\ell rT}{\de}}^3$. Hence
\begin{align}
\ve{\Ga^{(\bullet)\rc 2}(h-h^*)}^2
&\le 2\pa{\ve{\Ga^{(\bullet)\rc 2}(h-h_{LS})}^2 + \ve{\Ga^{(\bullet)\rc2}(h_{LS}-h^*)}^2}\\
&\le 2\pa{\ve{Q^{(\bullet)-\rc 2}\Ga^{(\bullet)} Q^{(\bullet)-\rc 2}}
\ve{Q^{(\bullet)\rc2}(h-h_{LS})}_2^2
+ \ve{\Ga^{(\bullet)\rc2}(h_{LS}-h^*)}^2
}\\
&\le C_9\pa{\rc{c\ell r T}r\pa{\ln\pf{c\ell rT}{\de}}^3 + \rc{c\ell T} \pa{\ln \pf{c\ell rT}{\de}}^3
}&\text{by~\eqref{e:Ga12delta-h}}\\
&\le \fc{C_9}{c\ell T}\pa{\ln \pf{c\ell rT}{\de}}^3
\end{align}
%
%
and similarly
\begin{align}
\ve{\Ga^{(j)\rc 2}\pa{\coltwo gh - \coltwo{g^*}{h^*}}}_2^2
&\le 
\fc{C_{10}}{\ell T} \pa{\ln \pf{\ell rT}{\de}}^4.
\label{e:Gaj-gh}
\end{align}
Now $\ve{\Ga^{(\bullet)\rc 2}(h-h^*)}_2^2$ represents the mean square estimation error when the input is 0 and the noise is $N(0,\si^2)$, so 
\begin{align}
\si\ve{(H-H^*)H^*_{\textrm{unr}}}_2&=
\ve{\Ga^{(\bullet)\rc 2}(h-h^*)}\le
 \fc{C_{9}}{\sqrt{c\ell T}}\pa{\ln \pf{c\ell rT}{\de}}^{\fc32}.
\end{align}
This establishes one-half of Theorem~\ref{t:main}.

We can decompose 
\begin{align}
M^{(j,k)}_{\cos,t} &=
\coltwo{x^{(j,k)}(t-1:t-r)}{\ol{\ol y}_{\cos}^{(j,k)}(t-1:t-r)}+\coltwo{0}{\ze^{(j,k)}_{\cos}(t-1:t-r)}
\end{align}
and similarly for $\sin$.
Define $M^{(j,k)}_t$ as follows: letting $y(t)$ be the response to $x(t) = e^{\fc{2\pi ijt}{cr}}$, $j\le \fc{cr}2$, with noise $\eta^{(j)}(t)=\eta^{(j)}_{\cos}(t)+i\eta^{(j)}_{\sin}(t)$, 
let $M^{(j,k)}_t = \coltwo{x(t-1:t-r)}{y(t-1:t-r)}$. We can decompose the mean response $\E M^{(j,k)}_t = \E [M^{(j,k)}_{\cos,t}+ iM^{(j,k)}_{\sin,t}]$. We obtain an upper bound on the difference in the square mean response:
\begin{align}
&
\ba{\coltwo gh - \coltwo{g^*}{h^*}}^\top 
(\E M^{(j,k)}_t)(\E M^{(j,k)}_t)^\top 
\ba{\coltwo gh - \coltwo{g^*}{h^*}}\\
&=
\ba{\coltwo gh - \coltwo{g^*}{h^*}}^\top 
\E [M^{(j,k)}_{\cos,t}- iM^{(j,k)}_{\sin,t}]\E[M^{(j,k)}_{\cos,t}+ iM^{(j,k)}_{\sin,t}]^{\top} 
\ba{\coltwo gh - \coltwo{g^*}{h^*}}\\
&\le 
\ve{\Ga^{(j)\rc 2}\pa{\coltwo gh - \coltwo{g^*}{h^*}}}_2^2\\
&\le 
\fc{C_{10}}{\ell T} \pa{\ln \pf{c\ell rT}{\de}}^4
\end{align}
using~\eqref{e:Gaj-gh}. 
Since the square mean response is exactly $\ab{[(G-G^*) + (H-H^*) H^*_{\textrm{unr}}G^*](e^{\fc{2\pi ij}{cr}})}^2$, we get 
\begin{align}\label{e:bd-all-freqs}
\ab{[(G-G^*) + z^{-1}(H-H^*) H^*_{\textrm{unr}}G](e^{\fc{2\pi ij}{cr}})}
&\le  
\fc{C_{10}}{\sqrt{\ell T}} \pa{\ln \pf{c\ell rT}{\de}}^{2}.
\end{align}
Note the same inequality holds for $j$ replaced by $r-j$ and $M^{(j,k)}_{\cos,t}+ iM^{(j,k)}_{\sin,t}$ replaced by $M^{(j,k)}_{\cos,t}- iM^{(j,k)}_{\sin,t}$, so~\eqref{e:bd-all-freqs} holds for all $j\in \Z$.

\subsection{Interpolation}

\begin{lem}\label{l:interp}
Let $Q(z):= \sumz k{r-1} a_k z^{k}$, where $a_k\in \C$. 
\begin{enumerate}
\item
\cite[Theorem 15.2]{trefethen2013approximation}
For any $N\ge r$, $\ve{Q}_\iy\le \pa{\fc{2}{\pi}\ln (r+1)+1} \max_{j=0,\ldots, N-1} |Q(e^{\fc{2\pi i j}{N}})|$.
\item
\citep{bhaskar2013atomic}
For any $N\ge 4\pi r$, $\ve{Q}_\iy \le \pa{1+\fc{4\pi r}{N}} \max_{j=0,\ldots, N-1} |Q(e^{\fc{2\pi i j}{N}})|$. 
\end{enumerate}
\end{lem}
From~\eqref{e:bd-all-freqs} we get that for $\ep=\fc{C_{10}}{\sqrt{\ell T}}\pa{\ln \pf{\ell r T}{\de}}^{2}$, $\om=e^{\fc{2\pi i}{cr}}$, $j\in \Z$, that
\begin{align}
\ab{[(G-G^*) +z^{-1} (H-H^*) G^* H^*_{\textrm{unr}}](\om^j)} &\le \ep\\
\implies
\ab{[(G-G^*)(1-z^{-1}H^*) + z^{-1} (H-H^*)G^*](\om^j)} &\le \ep \ab{1-\om^{-j}H^*(\om^j)} \le \ep(1+\ve{H^*}_\iy).
\end{align}
Suppose $c>8\pi$. 
By Lemma~\ref{l:interp}, since $(G-G^*)(1-z^{-1}H^*) + z^{-1}(H-H^*)G^*$ has degree $\le 2r$ in $z^{-1}$, 
\begin{align}
\ve{(G-G^*)(1-z^{-1}H) + z^{-1}(H-H^*)G^*}_\iy &\le \ep\pa{1+\fc{8\pi}{c}} (1+\ve{H^*}_\iy)\\
\implies
\ve{(G-G^*) + z^{-1}(H-H^*)G^*H^*_{\textrm{unr}}}_\iy &\le \ep\pa{1+\fc{8\pi}{c}} (1+\ve{H^*}_\iy)\ve{H^*_{\textrm{unr}}}_{\iy}.
\end{align}

This finishes the proof of Theorem~\ref{t:main}.

\subsection{Truncation error}

We need the following lemma.
\begin{lem}\label{l:cov-lb}
Let $F(z)=\sumz t{\iy} f(t)z^{-t}$, $f(0)=1$ and $G(z) = \rc{F(z)} = \sumz t{\iy} g(t)z^{-t}$. Let $K=\pa{\sumz t{r-1}|g(t)|}^2$. Then letting $f(t)=0$ for $t<0$, $\sumo tr f(t-r:t-1)f(t-r:t-1)^\top \succeq \rc{K^2}I_r$.
\end{lem}
\begin{proof}
For any power series $F(z)=\sumz t{\iy} f(t)z^{-t}$, define $Z_F\in \R^{d\times d}$ by $(Z_F)_{i,j}=f(i-j)$. (Here, $f(i)=0$ for $i<0$.) Note that $Z_FZ_G=Z_{FG}$. Let $A=Z_FZ_F^\top = \sumo tr f(t-r:t-1)f(t-r:t-1)^\top$. From $F(z)G(z)=1$ we get $Z_GZ_F=I_d$, hence $Z_GAZ_G^\top = Z_GZ_FZ_F^\top Z_G^\top=I_d$. Because $Z_G$ is invertible, we have  $A\succeq \la I_d$ iff $Z_G(A-\la I_d)Z_G^\top\succeq 0$. Now $Z_G(A-\la I_d)Z_G^\top = I-\la Z_GZ_G^\top$.
Letting $B=I-\la SS^\top$, we have
\begin{align}
B_{ii} - \sum_{j\ne i} B_{ij} &= 1 - \la \sum_{j,k} S_{ik}S_{jk} \ge 1-\la K^2\ge 0.
\end{align}
Thus by Gerschgorin's Disk Theorem, all eigenvalues of $B$ are $\ge 0$.
\end{proof}

\begin{proof}[Proof of Theorem~\ref{t:main-trunc}]
The proof of Theorem~\ref{t:main-trunc} relies on the following simple fact: 
If $D_1,D_2$ are two distributions on $\Om$ with TV-distance $\le \de$, and $\cal A$ is any algorithm with input space $\Om$, then $\cal A(x), x\sim D_1$ and $\cal A(x), x\sim D_2$ also have TV-distance $\le \de$.

Consider Algorithm~\ref{a:ls} run with signals $x_{\iy}$ stretching back to $-\iy$ and signals $x_{\ge -L}$ only stretching back to $-L$. Consider the distributions they induce on $y(1:T)$. Suppose we choose $L$ so that the TV-distance between those distributions is $\le \de':=\fc{\de}{8c\ell r}$. 
Because there are $<4c\ell r$ independent rollouts, the total TV-distance is $\le\fc{\de}{2}$. Then we can apply Theorem~\ref{t:main} with $\de \mapsfrom \fc{\de}2$ to get the desired result.

Let $y_\iy$ and $y_{\textrm{fin}}$ be the output signals given input signals $x_{\iy}$ and $x_{\ge -L}$, and noise $\eta_{\iy}$ and $\eta_{\ge-L}$. We have (using the shorthand $f_P:=f\one_{P}$)
\begin{align}
y_\iy(t+1) &= \hunrs * g^* * x_\iy(t) + \hunrs * \eta_\iy(t+1)\\
y_{\textrm{fin}}(t+1) &=\hunrs * g^**(x_\iy \one_{\ge -L}) (t) + [\hunrs*(\eta_\iy\one_{\ge -L})](t+1)\\
&=[(\hunrs * g^*)_{\le L+t}*x_\iy](t) + (h_{\textrm{unr},\le L+t+1}^* *\eta_\iy)(t+1)\\
y_\iy(t+1)-y_{\textrm{fin}}(t+1)&=
[(\hunrs * g^*)_{>L+t}*x_\iy](t) + (h_{\textrm{unr},> L+t+2}^* *\eta_\iy)(t+1)
\end{align}

To calculate the TV distance between the distributions of $y_\iy(1:T)$ and $y_{\textrm{fin}}(1:T)$, we need to bound the difference between the means and covariances.

\paragraph{Bounding difference in means.} Note for $t\ge 0$, by the assumption $L\ge R_{H_{\textrm{unr}}^*G^*}(\ep_2)-1$ and Lemma~\ref{l:suff-l}, we have 
\begin{align}
[(\hunrs*g^*)_{> L+t}*x_\iy](t)&
\le \ve{(\hunrs*g^*)_{\ge L+1}}_1 \le \ep_2 
\end{align}
so $\ve{\E (y_\iy-y_{\textrm{fin}})(1:T)}\le \ep_2\sqrt T$. 

\paragraph{Bounding difference in covariances.}
Because $\E[ \eta_\iy(i)\eta_\iy(j) ]= \one_{i=j}$, 
\begin{align}
\Cov[y_\iy(1:T)]_{i,j} &= \E [y_\iy(i) y_\iy(j)]\\
&=\E [(\hunrs*\eta_\iy)(i) (\hunrs*\eta_\iy)(j)]\\
&=\E \ba{\sum_{k=-\iy}^{\min\{i,j\}} \hunrs(i-k) \hunrs(j-k)}
\end{align}
so
\begin{align}
\Cov[y_\iy(1:T)] &=\sumo j\iy \hunrs(j-T:j-1)\hunrs(j-T:j-1)^\top.
\end{align}
Similarly
\begin{align}
\Cov[y_{\textrm{fin}}(1:T)] &=
\sumo j\iy h_{\textrm{unr},\le L+t+1}^*(j-T:j-1)h_{\textrm{unr},\le L+t+1}^*(j-T:j-1)^\top
\end{align}
Let $K=\pa{1+\sumz t{T-2}|h^*(t)|}^2$. 
When $L+t+2\ge T$, by Lemma~\ref{l:cov-lb} we can lower-bound this by
\begin{align}
\Cov[y_{\textrm{fin}}(1:T)] &\succeq \sumo j{L+t+2} \hunrs(j-T:j-1)\hunrs(j-T:j-1)^\top\succeq \rc{K^2}I_T
\end{align}
Also,
\begin{align}
\Cov[y_\iy(1:T)] -\Cov[y_{\textrm{fin}}(1:T)] 
&\preceq \sum_{j=L+T+2}^{\iy} \hunrs(j-T+1:j)\hunrs(j-T+1:j)^\top\\
&\preceq \pa{\sum_{j=L+T+2}^\iy \ve{\hunrs(j-T+1:j)}^2}I_T\\
&\preceq T\pa{\sum_{j=L+2}^{\iy} \hunrs(j)^2} I_T\\
&\preceq T\pa{\sum_{j=L+2}^\iy |\hunrs(j)|}^2I_T\le T\ep_1^2 I_T
\end{align}
where in the last inequality we used the assumption $L\ge R_{H^*_{\textrm{unr}}}(\ep_1)-2$ (for the $\ep_1$ we will choose) and Lemma~\ref{l:suff-l}.

\paragraph{Bounding TV distance.}
For a random variable let $\cal D(X)$ denote its distribution. 
We apply the following formula for KL-divergence,
\begin{align}
d_{KL}(N(\mu_1,\Si_1)||N(\mu_2,\Si_2))
&=\rc2 \ba{\ln \fc{|\Si_1|}{|\Si_2|} - d + \Tr(\Si_1^{-1}\Si_2) + (\mu_1-\mu_2)^\top \Si_1^{-1}(\mu_1-\mu_2)},
\end{align}
for $\cal D(y_{\textrm{fin}}(1:T))=N(\mu_1,\Si_1)$ and $\cal D(y_\iy(1:T))=N(\mu_2,\Si_2)$.
Here,  $\Si_1\succeq \rc{K^2}I_T$ and $\Si_2-\Si_1\preceq T\ep_1^2 I_T$, so
\begin{align}
d_{KL}(\cal D(y_{\textrm{fin}}(1:T))||\cal D(y_\iy(1:T)))
&\le \rc 2\ba{
T\ln \pf{1/K^2}{1/K^2+\ep}-T + T(1+K^2T\ep^2) + K^2T\ep_2^2}\\
&\le \rc 2 (K^2T^2 \ep_1^2 + K^2T\ep_2^2).
\end{align}
Now choose $\ep_1 = \sfc{\de^{\prime2}}{2T^2K^2}$ and $\ep_2 = \sfc{\de^{\prime2}}{2TK^2}$ to get this is $\le\fc{\de^{\prime2}}2$. Then by Pinsker's inequality, 
\begin{align}
d_{TV}(\cal D(y_{\textrm{fin}}(1:T)),\cal D(y_\iy(1:T)))
&\le \sqrt{\rc 2 \cdot \fc{\de^{\prime2}}2} = \fc{\de'}2.
\end{align}
This gives the desired result, noting that the assumption $L \ge \max\bc{
R_{H_{\textrm{unr}}^*}\pa{\fc{\de}{4KT\sqrt{c\ell r}}},
R_{H_{\textrm{unr}}^*G^*}
\pa{\fc{\de}{4K\sqrt{c\ell rT}}}}$ does indeed imply that the inequalities for $L$ are indeed satisfied for the values of $\ep_1$, $\ep_2$ and $\de'=\fc{\de}{8c\ell r}$ we chose.
Thus the TV-distance between the $y(1:T)$ of all the rollouts is at most $\fc{\de}2$, as needed.
\end{proof}

\section{Conclusion and further directions}

In the regime where Theorem~\ref{t:main-trunc} applies, we expect the dependence on the number of samples, as well as on $\ve{H_{\textrm{unr}}^*}$, to be optimal. 
However, note that our theorem requires at least $\Om(r^2)$ rollouts. It is an interesting question whether the bounds hold for fewer rollouts, or even for one rollout, with carefully designed inputs, analogous to results in the case of LDS without hidden state~\cite{simchowitz2018learning}.
Another open question is to prove a lower bound for the number of samples, in terms of $\ve{H_{\textrm{unr}}^*}$.

By improperly learning the Kalman filter as an autoregressive model, we incur sample complexity depending on $\sqrt{r}$ rather than $\sqrt{d}$, where $d$ is the dimension of the hidden state; obtaining bounds depending on $d$ seems to be a difficult problem. More generally, one can also consider optimal filtering for other noise models (where the Kalman filter is no longer optimal).

We expect that the theorem can be generalized in a straightforward manner to multiple-input, multiple-output systems. 

Finally, one can complete the ``identify-then-control'' pipeline by using the estimates from our algorithm for robust control. Although estimation in $\cal H_\iy$ norm of non-strictly stable systems is not possible in our setup, non-stable systems often arise in practice, so it is of great interest to find a weaker guarantees for such systems, perhaps under further assumptions, that still allow for robust control.

\printbibliography
\appendix

\pagebreak

\section{Notation}
\label{s:notation}


\begin{tabular}{|c|c|}
\hline 
Notation & Definition\tabularnewline
\hline 
$H_{\textrm{unr}}^*(z)$ & $\rc{1-z^{-1}H^*(z)}$\tabularnewline
\hline 
$x^{(\bullet,k)}=x^{(\bullet)}$ & $\mathbf{0}$ (the zero signal)\tabularnewline
\hline 
$x_{\cos}^{(j,k)}=x_{\cos}^{(j)}$ & $t\mapsto \cos\pa{\frac{2\pi jt}{cr}}$\tabularnewline
\hline 
$x_{\sin}^{(j,k)}=x_{\sin}^{(j)}$ & $t\mapsto \sin\pa{\frac{2\pi jt}{cr}}$\tabularnewline
\hline 
$y^{(\bullet,k)},y_{*}^{(j,k)}$ ($*=\cos,\sin$) & Outputs for the above inputs\tabularnewline
\hline 
$\ol y^{(\bullet,k)},\ol y_{*}^{(j,k)}$ & Expected value given $y(s),x(s)$ for $s<t$\tabularnewline
\hline 
$\eta^{(\bullet, k)}, \eta_{*}^{(j,k)}$ & $N(0,\si^2)$ noise in the rollouts; $y^{(j, k)}_*=\ol y^{(j, k)}_*+\eta^{(j, k)}_*$ \tabularnewline
\hline 
$\ol{\ol y}_{*}^{(j)}$ & Expected value given only $x$\tabularnewline
\hline 
$\zeta_{*}^{(j,k)}$ & Accumulated noise for the inputs, $y^{(j,k)}_* = \ol{\ol y}^{(j,k)}_* + \zeta^{(j,k)}_*$ \tabularnewline
\hline 
$M^{(\bullet,k)}$ & Matrix with columns $y^{(\bullet,k)}(t-1:t-r)$, $1\le t\le T$ \tabularnewline
\hline 
$M_{*,t}^{(j,k)}$  & Matrix with columns $ \coltwo{x^{(j)}_{*}(t-1:t-r)}{y^{(j, k)}_{*}(t-1:t-r)}$, $1\le t\le T$\tabularnewline
\hline
$\mathbf{X}_{*}^{(j)}$ & Matrix with columns $x^{(j)}_{*}(t-1:t-r)$, $1\le t\le T$\tabularnewline
\hline 
$\mathbf{Y}_{*}^{(j,k)}$ & Matrix with columns $y^{(j,k)}_{*}(t-1:t-r)$, $1\le t\le T$\tabularnewline
\hline
$\mathbf{Z}_{*}^{(j,k)}$ & Matrix with columns $\ze^{(j,k)}_{*}(t-1:t-r)$, $1\le t\le T$\tabularnewline
\hline 
$x^{(j)}$ & $t\mapsto e^{\fc{2\pi i j t}{cr}}$\tabularnewline
\hline
$\eta^{(j)}$ & $\eta^{(j)}(t)=\eta^{(j)}_{\cos}(t)+i\eta^{(j)}_{\sin}(t)$, $\eta^{(j)}_{\cos}(t), \eta^{(j)}_{\sin}(t)\sim N(0,\si^2)$\tabularnewline
\hline 
$M^{(j)}$ & Matrix with columns $\coltwo{x^{(j)}(t-1:t-r)}{y^{(j)}(t-1:t-r)}$, $1\le t\le T$\tabularnewline
\hline
$h_{LS}$ & Solution to~\eqref{e:min-max-h}\tabularnewline
\hline 
$g_{LS}^{(j)}$ & Solution to~\eqref{e:min-max-g}\tabularnewline
\hline 
$g,h$ & Solution to~\eqref{e:ls}\tabularnewline
\hline 
$\Gamma^{(\bullet)}$ & $\E_{\eta^{(\bullet,k)}} M^{(\bullet, k)}M^{(\bullet, k)\top}$ \tabularnewline
\hline 
$\Gamma_{*,t}^{(j)}$ & $\E_{\eta^{(j,k)}_*}M^{(j,k)}_{*,t} M^{(j,k)\top}_{*,t}$ \tabularnewline
\hline 
$\Gamma_{X,*,t}^{(j)}$ & $\mathbf X^{(j)}_{*,t}\mathbf X^{(j)\top}_{*,t}$ \tabularnewline
\hline 
$\Gamma^{(j)}$ & $\Ga^{(j)}_{\cos, t} + \Ga^{(j)}_{\sin, t}$ \tabularnewline
\hline 
$\Gamma_{X}^{(j)}$ & $\Ga^{(j)}_{X,\cos, t} + \Ga^{(j)}_{X,\sin,t}$ \tabularnewline
\hline 
$Q^{(\bullet)}$ & $ \sumo k{c\ell r} M^{(\bullet,k)} M^{(\bullet,k)\top}$ \tabularnewline
\hline 
$Q^{(j)}$ & $\sumo k{\ell} (M^{(j,k)}_{\cos}M^{(j,k)\top}_{\cos} + M^{(j,k)}_{\sin}M^{(j,k)\top}_{\sin})$ \tabularnewline
\hline 
$P_{X}^{(j)}$ & Projection onto column space of $\Ga^{(j)}_X$ \tabularnewline
\hline 
$L^{(\bullet)}(h)$ & $\rc{r}\sumo k{c\ell r} 
\ve{M^{(\bullet, k)\top}(h-h_{LS})}^2$ \tabularnewline
\hline 
$L^{(j)}(g,h)$ & $\sumo k{\ell}
\ba{ \ve{M^{(j,k)\top}_{\cos}\coltwo gh - y^{(j,k)}_{\cos}}^2
+ \ve{M^{(j,k)\top}_{\sin}\coltwo gh - y^{(j,k)}_{\sin}}^2}$ \tabularnewline
\hline 
\end{tabular}

\section{Learning FIR is not adequate}
\label{a:fir}

Consider the system 
\begin{align*}
h(t)&=rh(t-1)+x(t-1)+\xi(t),\\
y(t)&=h(t)+\eta(t)
\end{align*}
where $0<r<1$ and $\xi(t), \eta(t)\sim N(0,1)$. Then we can calculate using formulas for the Kalman filter that the variance in the estimation of $h$ and $y$ are $\si_h^2 = \fc{r^2+\sqrt{r^4+4}}2$, and $\si_y^2 = \si_h^2+1$. The average squared error in estimating $y_t$ using the Kalman filter is $\si_y^2$, which remains finite as $r\to 1$. On the other hand, if we were to estimate $y_t$ without using the previous observations $y_{t-1},\ldots$, then the average estimation error is $1+(1+r^2+r^4+\cdots) = 1+\rc{1-r^2}$, which blows up as $r\to 1$. Hence the multiplicative factor between the error using a FIR filter, and using the optimal filter, goes to $\iy$ as $r\to 1$. (This kind of ratio is exactly what FIR methods suffer; see for example \cite[\S 3 (Process noise)]{tu2017non}. Their bounds depend on $\ve{G}_\iy$, which is $\rc{1-r^2}$ in this example.)


%
%
\end{document}